\documentclass{article}

 \usepackage[preprint]{neurips_2026}

\usepackage[utf8]{inputenc} 
\usepackage[T1]{fontenc}    
\usepackage{hyperref}       
\usepackage{url}            
\usepackage{booktabs}       
\usepackage{amsfonts}       
\usepackage{nicefrac}       
\usepackage{microtype}      
\usepackage{xcolor}         

\usepackage{amsmath}
\usepackage{amssymb}
\usepackage{mathtools}
\usepackage{amsthm}
\usepackage{mathrsfs}

\usepackage{bm,bbm}
\usepackage{xspace}
\usepackage{multirow}
\usepackage{cancel}
\usepackage[normalem]{ulem}
\usepackage{enumitem}
\usepackage{caption}
\usepackage{booktabs}

\usepackage[ruled,vlined,linesnumbered]{algorithm2e}

\usepackage{subcaption}
\usepackage{graphicx}
\usepackage{threeparttable, array, float}
\usepackage{wrapfig}

\usepackage{tabularx}
\usepackage[table]{xcolor}
\usepackage{array}
\usepackage{booktabs}
\usepackage{threeparttable}
\usepackage{circledsteps}

\usepackage{tocloft}

\definecolor{mydarkblue}{rgb}{0,0.08,0.45}
\definecolor{MyCrimson}{RGB}{220, 20, 60}
\hypersetup{
    colorlinks=true,
    linkcolor=mydarkblue,
    citecolor=mydarkblue,
    urlcolor=mydarkblue
}

\usepackage[capitalize,noabbrev]{cleveref}
\usepackage{pifont}
\usepackage{tcolorbox}
\usepackage{nicematrix}

\usepackage[dvipsnames]{xcolor}
\newcommand{\cmark}{\textcolor{ForestGreen}{\ding{51}}}
\newcommand{\xmark}{\textcolor{Maroon}{\ding{55}}}

\theoremstyle{plain}
\newtheorem{theorem}{\bf Theorem}
\newtheorem{lemma}{\bf Lemma}
\newtheorem{corollary}{\bf Corollary}

\theoremstyle{remark}

\newtheorem{definition}{\bf Definition}

\newtheorem{assumption}{\bf Assumption}

\crefname{assumption}{Assumption}{Assumptions}
\Crefname{assumption}{Assumption}{Assumptions}

\Crefname{condition}{Condition}{Conditions}
\Crefname{proposition}{Proposition}{Proposition}

\DeclareMathOperator*{\argmin}{arg\,min}
\newcommand{\alg}{$\mathsf{PROBE}~$}
\newcommand{\algn}{$\mathsf{PROBE}$}

\title{To Solve Bilevel Optimization with Nonconvex Lower Levels, We Need Second-Order Stationarity}

\author{%
\textbf{Zhiyao Zhang\textsuperscript{1}, Menglu Yu\textsuperscript{2}, Alvaro Velasquez\textsuperscript{3}},
\textbf{Nathaniel D. Bastian\textsuperscript{4}, Jia Liu\textsuperscript{1}} \\
\textsuperscript{1}The Ohio State University\,\,\,
\textsuperscript{2}Meta\,\,\,
\textsuperscript{3}University of Colorado Boulder\,\,\,
\textsuperscript{4}Johns Hopkins University \\
\texttt{zhang.15178@osu.edu, liu@ece.osu.edu}
}

\begin{document}

\maketitle

\begin{abstract}
Although bilevel optimization (BLO) has emerged as a powerful framework for addressing many complex and nested machine learning problems in recent years, most existing studies are confined to the lower-level strongly convex (LLSC) or lower-level generally convex (LLGC) settings (i.e., the lower-level objective function is assumed to be, at least, convex).
While the LLSC/LLGC assumptions render more tractable algorithmic design and theoretical analysis, they are too rigid to encompass many machine learning problems in practice.
The limitations of LLSC/LLGC assumptions in BLO motivate us to investigate solving the BLO problem in the general lower-level nonconvex (LLNC) settings, which remains in its infancy.
In the literature on LLNC-BLO, most of the existing works either require additional structures in the lower-level objective function for tractable theoretical analysis, or adopt the first-order stationarity reformulation as a lower-level surrogate problem, which is inherited from the LLSC/LLGC settings but could lose their effectiveness in the LLNC setting.
To bridge this gap, in this work, we propose to reformulate the nonconvex lower-level problem using a {\em second-order stationarity-based surrogate,} the solution of which guarantees a local optimal solution at the lower level.
Based on this reformulation, we propose the \alg (\underline{P}erturbed g\underline{r}adient alg\underline{o}rithm for \underline{b}ilevel probl\underline{e}m) and show that it overcomes the limitations of prior works by probing and escaping lower-level saddle points.
We theoretically prove that \alg achieves a finite-time convergence rate of $\mathcal{O}(T^{-\frac{2}{5}})$, where $T$ denotes iterations.
To our knowledge, this work is the first to establish the finite-time convergence rate guarantee for achieving lower-level second-order stationary solutions in general LLNC-BLO.
Our experiments on both a large language model-based data curation task and a meta-learning task also show that \alg outperforms state-of-the-art methods.
\end{abstract}

\section{Introduction}

In recent years, bilevel optimization (BLO)~\citep{bracken1973mathematical,liu2021investigating,zhang2024introduction} has received significant attention in the machine learning community due to the rise of a wide range of complex and nested machine learning problems, such as reinforcement learning \citep{hong2023two,kudo2026sample}, meta-learning \citep{franceschi2018bilevel,ji2021bilevel}, adversarial training \citep{zhang2022revisiting}, large language model (LLM) fine-tuning \citep{shen2024seal}, to name just a few.
In general, a BLO problem can be formulated as follows:
\begin{equation}\label{eq:BLO}
    \begin{aligned}
        \mathrm{BLO:} \,\, \min_{x\in\mathbb{R}^p,y\in\mathbb{R}^q} f(x,y) \hspace{3em}
        \text{subject to } &\,\, y \in \mathcal{S}(x) := \argmin_y g(x,y),
        \vspace{-0.5em}
    \end{aligned}
\end{equation}
where $f(x,y)$ and $g(x,y)$ are the upper- and lower-level objective functions, respectively.
Clearly, the challenge of solving a BLO problem stems from the nested structure, i.e., part of the decision variables $y$ in the upper-level objective $f(x,y)$ is obtained from the set of optimizers $S(x)$ given an upper-level variable $x\in\mathbb{R}^p$.
To address this challenge and for tractable algorithmic design and analysis, most existing works in the BLO literature have relied on the \textit{restrictive} lower-level strong convexity (LLSC) assumption, i.e., $g(x,\cdot)$ is \textit{strongly convex} with respect to (w.r.t.) $y$ for any given $x\in\mathbb{R}^p$ \citep{ghadimi2018approximation,ji2021bilevel,yang2021provably,dagreou2022framework}, thereby ensuring a singleton solution set for the lower-level problem and well-defined upper-level hypergradient through the implicit function theorem.
To address the limitation of the stringent LLSC assumption, a recent line of works in the BLO literature relaxes LLSC to lower-level general convexity (LLGC), i.e., $g(x,\cdot)$ is \textit{convex} w.r.t. $y$ \citep{cao2023projection,jiang2023conditional,liu2023averaged}.
Unfortunately, LLGC remains highly restrictive, since most modern learning models are based on deep neural networks, which are highly {\em nonconvex} and violate the LLGC assumption.

To date, research on bilevel optimization (BLO) problems with nonconvex lower levels (LLNC) remains in its infancy, and existing results are limited. To the best of our knowledge, early attempts to address LLNC-BLO \citep{kwon2023penalty, shen2023penalty, xiao2023generalized, chen2024finding, liu2024moreau, jiang2025correspondence, jiang2025discretization, ma2026sun} have achieved only partial success and exhibit various limitations. 
One key reason for this stagnation is that, under LLNC, finding an optimal solution $y^*(x)$ to the lower-level problem is already intractable in general (typically NP-hard). 
This difficulty is further compounded by the potential non-uniqueness of $y^*(x)$ and the challenge of establishing well-defined upper-level hypergradients given a $y^*(x)$. 
Consequently, the aforementioned works either impose specific structural assumptions or consider relaxations of the LLNC-BLO problem.
For example, some existing works \citep{kwon2023penalty, jiang2025correspondence, jiang2025discretization} assume bounded domains for the variables (either $x$, $y$, or both), which limits their practical applicability.
Other works relax the LLNC-BLO formulation by replacing the lower-level optimality condition $y \in \arg\min_{y} g(x,y)$ with the first-order stationarity condition, i.e., seeking $y$-solutions that satisfy $\nabla_y g(x,y)=0$.
Although some studies further assume that the lower-level objective $g(x,y)$ satisfies the Polyak–Łojasiewicz (PL) condition \citep{shen2023penalty, xiao2023generalized, liu2024moreau, ma2026sun}, thereby ensuring that any first-order stationary point (FOSP) is globally optimal.
However, the PL condition is restrictive, since nonconvex problems may generally exhibit flat regions, saddle points, or suboptimal local minima, under which the PL condition fails to hold.

In this work, we argue that algorithmic designs using the FOSP relaxation as a surrogate for the lower-level problem could be easily trapped in undesirable saddle points of the lower-level objective (i.e., points that are FOSP but not local minima), thereby yielding poor solutions to the overall LLNC-BLO problem.
For example, consider the simple LLNC-BLO problem in Fig.~\ref{fig:EXAMPLE} where $x,y\in\mathbb{R}$, $f(x,y) = (x-2)^2 + y^2$, $g(x,y) = \frac{1}{4}y^4 - \frac{1}{2}xy^2$.
It can be readily verified that the PL condition fails to hold even for this simple example.
Consequently, sequences generated by existing algorithms based on the FOSP surrogate converge only to {\em saddle points} of $g(x,\cdot)$ (cf. Fig.~\ref{fig:contour}), leading to solutions that are significantly suboptimal compared to the true optimum (cf. Fig.~\ref{fig:gap}).

\begin{figure}[t]
    \centering
    \begin{subfigure}[t]{0.47\textwidth}
        \centering
        \includegraphics[width=\textwidth]{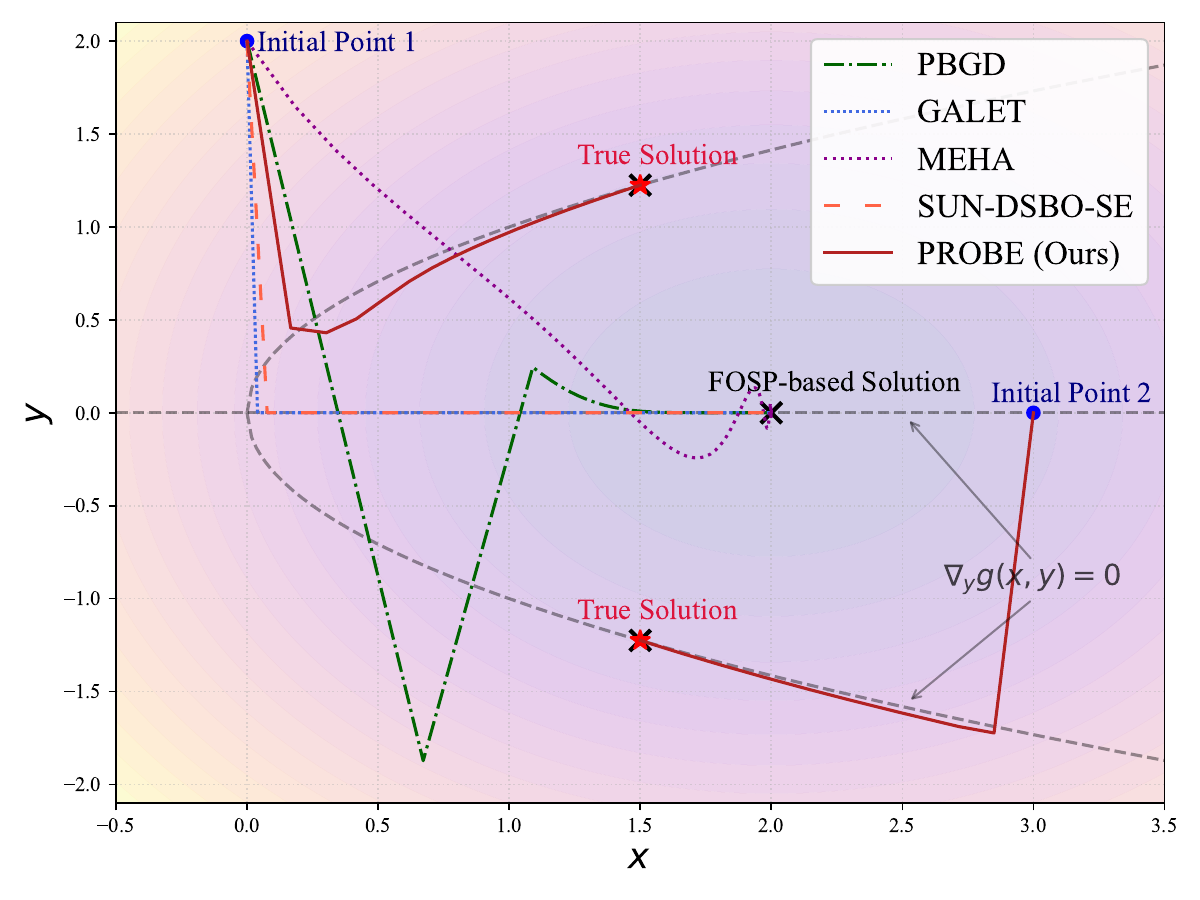}
        \vspace{-2em}
        \caption{Converging trajectories.}
        \label{fig:contour}
    \end{subfigure}
    \hfill
    \begin{subfigure}[t]{0.47\textwidth}
        \centering
        \includegraphics[width=\textwidth]{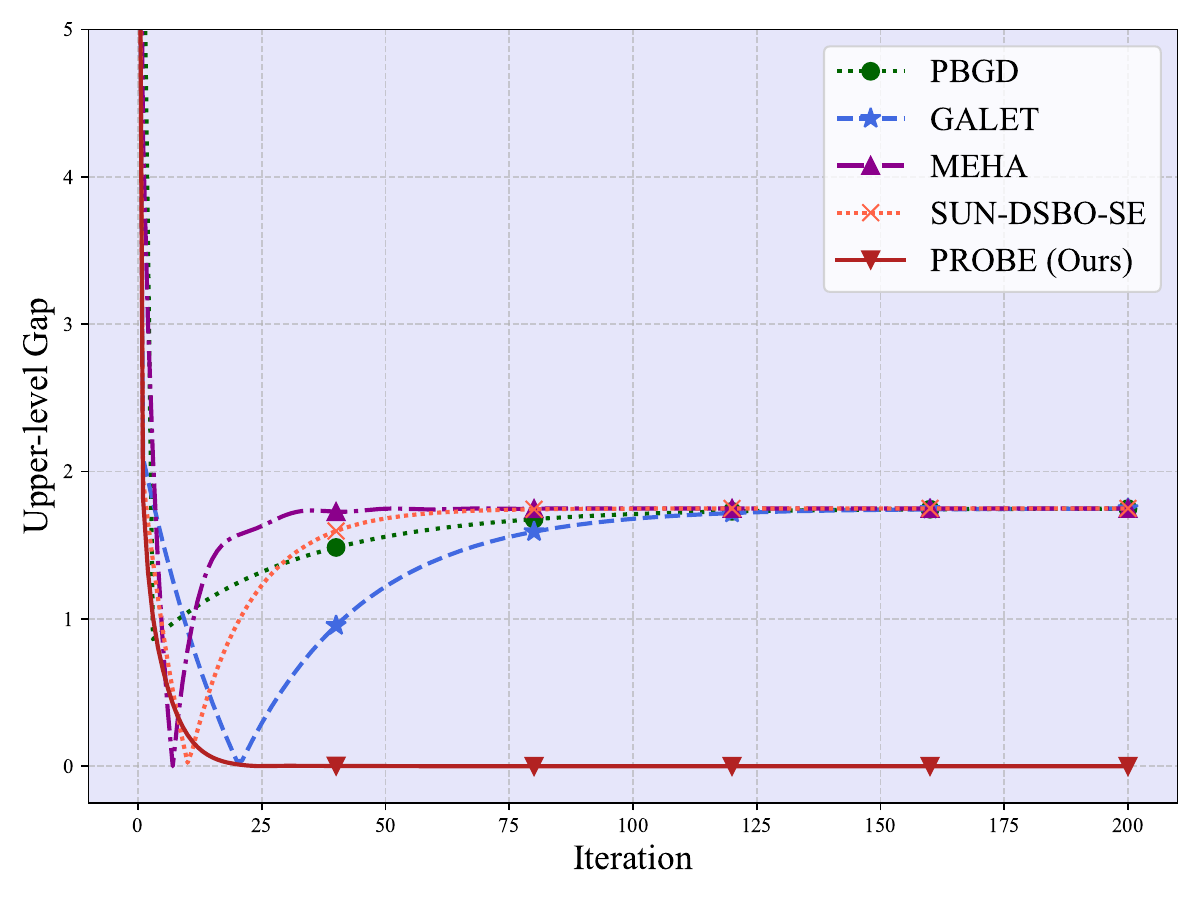}
        \vspace{-2em}
        \caption{Convergence gaps in (a).}
        \label{fig:gap}
    \end{subfigure}
    \caption{Example: $x,y\in\mathbb{R}$, $f(x,y) = (x-2)^2 + y^2$, $g(x,y) = \frac{1}{4}y^4 - \frac{1}{2}xy^2$.
    Fig.~\ref{fig:contour} shows the contour plot of $f(x,y)$ along with converging trajectories with initial points $(0,2)$ and $(3,0)$, and Fig.~\ref{fig:gap} plots $|f(x_t,y_t)-f^*|$ vs. iteration $t$, illustrating how far the convergent sequences are from solving the LLNC-BLO problem.
    Our \alg (Alg.~\ref{alg:main}) is the \textbf{only} method that \textbf{escapes the saddle point} and \textbf{finds true solutions}, compared to existing baselines methods: \textup{PBGD} \citep{shen2023penalty}, \textup{GALET} \citep{xiao2023generalized}, \textup{MEHA} \citep{liu2024moreau}, and \textup{SUN-DSBO-SE} \citep{ma2026sun}.}
    \label{fig:EXAMPLE}
    \vspace{-.2in}
\end{figure}

To address the limitations of FOSP-based lower-level surrogates, in this paper, we propose using {\bf second-order stationarity} as a surrogate for LLNC-BLO problems. 
Mathematically, second-order stationarity requires not only $\nabla_y g(x,y)=0$, but also that the Hessian with respect to $y$ is positive semidefinite (PSD), i.e., $\lambda_{\min}(\nabla_{yy}^2 g(x,y)) \ge 0$.
Our rationale for adopting this second-order-stationarity-based lower-level surrogate (SOS-LLS) is that a PSD Hessian enforces {\em local convexity} in the lower-level problem, thereby ruling out undesirable saddle points and guaranteeing at least local optimality with respect to $y$. 
As a concrete example, Fig.~\ref{fig:EXAMPLE} shows that our \alg method, based on SOS-LLS, successfully recovers true optimal solutions to the LLNC-BLO problem.

\begin{table}[t]
    \centering
    \caption{Comparison of \alg (Alg.~\ref{alg:main}) with state-of-the-art LLNC-BLO methods. 
    $\dagger$: \textup{BLO Setting} states the assumptions on the lower level of the LLNC-BLO.
    $\ddagger$: \textup{Surrogate Problem} refers to the relaxed surrogate of the lower level of the LLNC-BLO problem. Moreau envelope-based surrogates reduce to first-order stationarity under smoothness conditions \citep{ma2026sun}.
    }
    \label{tab:compare}
    \resizebox{0.9\textwidth}{!}{
    \begin{NiceTabular}{l c c c}
        \toprule
        \textbf{Algorithm} & \textbf{BLO Setting}$^\dagger$ & \textbf{Surrogate Problem}$^\ddagger$ & \textbf{Escape Saddle Points} \\
        \midrule
        PBGD & LLNC & First-order stationarity & \xmark \\
        GALET & LLNC & First-order stationarity & \xmark \\
        MEHA & LLNC & Moreau envelope & \xmark \\
        SUN-DSBO-SE & LLNC & Moreau envelope & \xmark \\
        \midrule
        \rowcolor{MyCrimson!20}
        \textbf{\alg (Ours)} & \textbf{LLNC} & \textbf{Second-order stationarity} & \textbf{\cmark} \\
        \bottomrule
    \end{NiceTabular}
    }
    \vspace{-1.5em}
\end{table}

Despite the significant advantage of avoiding undesirable saddle points, solving LLNC-BLO problems through the SOS-LLS reformulation remains largely underexplored, and both algorithm design and analysis are highly nontrivial. 
This is primarily due to two fundamental challenges: {\bf (i)} the SOS-LLS reformulation leads to a constrained problem that is inherently non-smooth, posing difficulties for algorithmic design; 
and {\bf (ii)} under the LLNC setting, the upper-level hypergradient is ill-defined, leaving no natural metric to quantify the upper-level stationarity gap.
In this paper, we address these challenges through the following key contributions:

\vspace{-0.75em}
\begin{list}{\labelitemi}{\leftmargin=1.5em \itemindent=-0.0em \itemsep=.1em}
    \item \textbf{Theoretical Foundations:} For the first time in the literature, we propose to incorporate second-order stationarity as the lower-level surrogate for LLNC-BLO problems, which allows to avoid undesirable saddle point solutions at the lower level without assuming the overly restrictive PL condition (cf. \Cref{tab:compare}).
    Building on the SOS-LLS reformulation, we introduce a new augmented convergence metric that leverages second-order stationarity to induce local curvature. 
    By progressively shrinking the augmentation parameter, the proposed metric recovers solutions to the original LLNC-BLO problem, thereby providing a principled foundation for both algorithm design and analysis.
    
    \item \textbf{Algorithmic Design and Analysis:} Based on the proposed SOS-LLS reformulation, we propose \alg (\underline{P}erturbed g\underline{r}adient alg\underline{o}rithm for \underline{b}ilevel probl\underline{e}m), a new algorithm that efficiently probes and guarantees escaping saddle points in the lower level of LLNC-BLO problems.
    Theoretically, we show that \alg achieves a finite-time convergence rate of $\mathcal{O}(T^{-\frac{2}{5}})$ at the upper level while guaranteeing local optimality at the lower level for solving LLNC-BLO problems, where $T$ denotes iteration counts.
    
    \item \textbf{Empirical Validation:} We validate the effectiveness of our \alg method through a dataset curation task for LLM fine-tuning and a meta-learning task, both of which can be formulated as a BLO problem.
    Our experimental results demonstrate that \alg outperforms existing baselines.
    We further conduct ablation studies to further reveal the significance of the key components in our proposed \alg method.
\end{list}

\section{Related Work}

In this section, we provide an overview on two related lines of research: (i) bilevel optimization (BLO), with a focus on methods that address BLO problems in the LLNC settings; and (ii) mechanisms for escaping saddle points in single-level nonconvex optimization theory for machine learning.

\textbf{1) Bilevel Optimization (BLO):} The study of BLO traces its roots back to as early as \cite{bracken1973mathematical, ye1995optimality}.
In recent years, BLO has received significant attention within the machine learning community, primarily driven by the growing necessity of solving nested-structured optimization problems in modern machine learning applications.
As mentioned earlier, most existing works in the BLO literature are based on the restrictive LLSC condition \citep{franceschi2018bilevel,ghadimi2018approximation,shaban2019truncated,arbel2021amortized,ji2021bilevel,yang2021provably,dagreou2022framework}, which significantly limits the practical applicability of the BLO framework.
Due to the limitations of LLSC, researchers have recently started to relax the LLSC assumption to the LLGC assumption \citep{cao2023projection,chen2023bilevel,jiang2023conditional,liu2023averaged,shen2024method}.
However, the convexity requirement in LLGC still excludes a wide range of practical scenarios.
The widening gap between the demand for solving general BLO problems in the LLNC setting and the limited theoretical understanding of this BLO setting motivates the studies of LLNC-BLO.
However, this area remains in its infancy, and existing works either impose strong structural assumptions or consider FOSP-based relaxations of the LLNC-BLO problem~\citep{kwon2023penalty, shen2023penalty, xiao2023generalized, chen2024finding, liu2024moreau, jiang2025correspondence, jiang2025discretization, ma2026sun}.
As shown in the illustrative example in \Cref{fig:EXAMPLE} and \Cref{tab:compare}, these methods fail to distinguish local minima from saddle points at the lower level, leading to solutions that are significantly suboptimal compared to the true optimum of the original LLNC-BLO problems.

\textbf{2) Escaping Saddle Points in Single-Level Nonconvex Optimization:} It is well-known that in the absence of convexity, traditional single-level optimization becomes significantly more challenging: finding a global optimum is NP-hard in general \citep{murty1987some}.
In fact, even finding a local minimum solution in nonconvex optimization is highly nontrivial.
One of the primary challenges stems from the ubiquity of saddle points, particularly in the landscape of high-dimensional objective functions \citep{dauphin2014identifying}, which is the typical regime where modern machine learning tasks reside.
Although second-order methods naturally tackle this issue by actively probing and escaping saddle points, their dimension-dependent convergence rates preclude them from practical large-scale applications.
In the machine learning literature, the question of how to escape saddle points in nonconvex optimization was first systematically studied in \cite{ge2015escaping}.
However, the convergence rate established in \cite{ge2015escaping} scales polynomially with the dimension.
Inspired by this breakthrough, a significant amount of research effort has been dedicated to designing algorithms that reduce the dimension dependence to scale at a polylog fashion \citep{agarwal2017finding,jin2017escape,allen2018neon2,fang2018spider,jin2018accelerated,liu2018adaptive,jin2021nonconvex,zhang2021escape}.
Another closely related line of research focuses on zeroth-order methods for escaping saddle points \citep{vlatakis2019efficiently, zhang2022zeroth, zhang2022faster, ren2023escaping}.
To our knowledge, the only works that incorporated saddle point escaping techniques in the BLO framework are \cite{huang2025efficiently, chen2025near}.
However, these studies are limited to the LLSC-based setting and only considered escaping saddle points at the upper level, which are {\em fundamentally different from} and {\em much simpler than} our setting, where we investigate saddle point escaping at the lower level for the LLNC setting.

\section{The Second-Order Stationarity-Based Lower-Level Surrogate}

In this section, we present the second-order stationarity-based lower-level surrogate reformulation for the LLNC-BLO problem.
We begin by introducing several standard smoothness conditions that are needed to formally define second-order stationary points:
\begin{assumption}[Smoothness of the Upper- and Lower-Level Objective Functions]\label{ass:smoothness}
    The upper- and lower-level objective functions $f$ and $g$ are twice differentiable, and there exist positive constants $\nu,\ell,\rho$, such that for any $z=(x,y)$ and $z'=(x',y')$, it holds that:
    \begin{align*}
        & |f(z)-f(z')| \le \nu \|z-z'\|, \,\,\,\,\,\, \|\nabla f(z)-\nabla f(z')\| \le \ell \|z-z'\|, \\
        & \|\nabla g(z)-\nabla g(z')\| \le \ell \|z-z'\|, \,\,\,\,\,\, \|\nabla^2 g(z)-\nabla^2 g(z')\| \le \rho \|z-z'\|,
    \end{align*}
    where $\|\cdot\|$ denotes Euclidean and spectral norm for vectors and matrices hereafter, respectively.
\end{assumption}
We note that the smoothness assumptions above are not only widely adopted in the BLO literature \citep{ghadimi2018approximation,ji2021bilevel,yang2021provably,xiao2023generalized}, but they have also been numerically observed and verified in large-scale modern machine learning models, such as LLMs \citep{castin2023smooth, malladi2023fine, li2024getting}.

Equipped with these smoothness properties, we now formally define stationary points, local minima, and saddle points for the lower-level of the LLNC-BLO problem.

\begin{definition}[First- and Second-Order Stationary Point]\label{def:stationary}
    Suppose \Cref{ass:smoothness} holds, and let $\lambda_{\min}(\cdot)$ denote the minimal eigenvalue of a matrix. For any given $x\in\mathbb{R}^p$ and any positive $\epsilon$,
    \begin{list}{\labelitemi}{\leftmargin=1.5em \itemindent=-0.0em \itemsep=-.2em}
        \item[1.] $y$ is a first-order stationary point (FOSP) if $\nabla_y g(x,y)=0$. It is an $\epsilon$-FOSP if $\|\nabla_y g(x,y)\| \le \epsilon$.
        \item[2.] $y$ is a second-order stationary point (SOSP) if $\nabla_y g(x,y)=0$ and $\lambda_{\min}(\nabla^2_{yy} g(x,y)) \ge 0$. It is an $\epsilon$-SOSP if $\|\nabla_y g(x,y)\| \le \epsilon$ and $\lambda_{\min}(\nabla^2_{yy} g(x,y)) \ge -\sqrt{\rho\epsilon}$.
    \end{list}
\end{definition}
We note that the definition of $\epsilon$-second-order stationary point here has also been adopted in the literature \citep{nesterov2006cubic,jin2017escape,jin2018accelerated,zhang2021escape}.

\smallskip
\begin{definition}[Local Minimum and Saddle Point]
    Suppose that $g$ is differentiable.
    For any given $x\in\mathbb{R}^p$, $y$ is a local minimum if $\nabla_y g(x,y)=0$ and there exists $\delta > 0$ such that $g(x,y) \le g(x,y')$ for all $y' \in \mathbb{B}_y(\delta)$, where $\mathbb{B}_y(\delta)$ denotes an open ball of radius $\delta$ centered at $y$.
    In contrast, $y$ is a saddle point if $\nabla_y g(x,y)=0$ and it is not a local minimum.
\end{definition}
When all saddle points are strict, i.e., $\lambda_{\min}(\nabla^2_{yy} g(x,y)) < 0$ holds for all saddle points, attaining a second-order stationary point is equivalent to reaching a local minimum.
In fact, this ``strict'' property has been extensively verified and documented for saddle points across various nonconvex landscapes (e.g., tensor decomposition or low-rank matrix recovery, see, e.g., \citep{ge2015escaping,bhojanapalli2016global,ge2016matrix,jin2017escape,sun2018geometric}, and see Appendix~\ref{sec:app-saddle} for more discussions.)
\footnote{While finding an SOSP does not theoretically preclude convergence to non-strict saddle points that could occur in highly overparameterized deep learning landscapes, it remains strictly far more advantageous than relying on FOSP. SOSP guarantees the evasion of all strict saddle points, including local maxima, by explicitly exploiting directions of negative curvature. Furthermore, because escaping non-strict saddle points generally requires computationally intractable higher-order derivatives, targeting an SOSP represents the strongest practically verifiable optimality condition. Empirically, the non-strict saddles satisfying SOSP conditions in deep networks often lie on flat manifolds that yield highly desirable generalization properties.}.

Given the NP-hardness of finding global optima for general nonconvex objective functions, it is standard practice to target local optimal solutions instead (i.e., points that satisfy the necessary conditions for global optimality).
However, as previously noted, simply relaxing the lower-level optimality constraint to merely finding an FOSP~\citep{shen2023penalty, xiao2023generalized, liu2024moreau, ma2026sun} is problematic, as it introduces the risk of converging to undesirable saddle points.
Fortunately, under the strict saddle assumption, identifying an SOSP yields local minima.
Thus, we propose to reformulate the LLNC-BLO problem with the following SOSP-based surrogate:
\begin{equation}\label{eq:surrogate}
    \min_{x\in\mathbb{R}^p,y\in\mathbb{R}^q} f(x,y), \hspace{2em}\text{subject to } \nabla_y g(x,y) = 0, \,\,\, \lambda_{\min}(\nabla^2_{yy} g(x,y)) \ge 0. 
\end{equation}

\section{The Proposed \alg Algorithm}

Based on the SOSP-based lower-level surrogate for LLNC-BLO problems, in this section, we will investigate how to design an efficient algorithm for solving  in Eq.~\ref{eq:surrogate}.
We start with defining an appropriate metric to measure the convergence error in designing algorithms for solving Eq.~\ref{eq:surrogate}.

Suppose that $(\tilde{x}, \tilde{y})$ is an $\epsilon'$-SOSP for $g(\tilde{x},\cdot)$ for some positive $\epsilon'$.
According to \Cref{def:stationary}, we have that $\nabla_{yy}^2 g(\tilde{x},\tilde{y}) + (\sqrt{\rho\epsilon'} + \sigma) I_q$ is positive definite for any arbitrarily small positive value $\sigma>0$ (which can be adaptively selected according to $\epsilon'$).
It then follows that the augmented objective function $g(\tilde{x}, y) + \frac{\mu}{2}\|y-\tilde{y}\|^2$ is locally $\frac{\sigma}{2}$-strongly convex within a sufficiently small neighborhood of $\tilde{y}$, where the parameter $\mu \triangleq \mu(\epsilon', \sigma) := \sqrt{\rho\epsilon'} + \sigma$.
By leveraging this local quadratic curvature, we can define the following $\mu$-\textit{augmented implicit function} with a small positive radius $r'=\frac{\sigma}{\rho}$:
\begin{align*}
    \Phi_{\mu}(\tilde{x}) := f(\tilde{x}, y_{\mu}^*(\tilde{x})), \,\,\, \text{where } y_{\mu}^*(\tilde{x}):=\argmin_{y\in\mathbb{B}_{\tilde{y}}(r')}\{ g(\tilde{x}, y) + \frac{\mu}{2}\|y-\tilde{y}\|^2 \},
\end{align*}
where the existence and uniqueness of $y_{\mu}^*(\tilde{x})$ can be guaranteed by the local strong convexity.
Indeed, the introduction of $\mu$-regularization restores a locally strongly convex curvature, thereby enabling the use of existing techniques from the LLSC-BLO setting, such as hypergradient-based methods \citep{ghadimi2018approximation,grazzi2020iteration,ji2021bilevel,huang2025efficiently}.
Consequently, following from the Implicit Function Theorem (IFT), the gradient of $\Phi_{\mu}(\tilde{x})$, referred to as the $\mu$-augmented hypergradient in this paper, can be computed as:
\begin{align*}
    \nabla \Phi_{\mu}(\tilde{x}) & = \nabla_x f(\tilde{x}, y_{\mu}^*(\tilde{x})) + \frac{d y_{\mu}^*(\tilde{x})}{dx} \nabla_y f(\tilde{x}, y_{\mu}^*(\tilde{x})) \\
    & = \nabla_x f(\tilde{x}, y_{\mu}^*(\tilde{x})) - \nabla_{xy}^2 g(\tilde{x}, y_{\mu}^*(\tilde{x})) \left[ \nabla_{yy}^2 g(\tilde{x}, y_{\mu}^*(\tilde{x})) + \mu I_q \right]^{-1} \nabla_y f(\tilde{x}, y_{\mu}^*(\tilde{x})).
\end{align*}

The $\mu$-augmented hypergradient naturally gives rise to the following convergence metric:

\begin{definition}[$\mu$-Augmented Hypergradient-Based Convergence Metric]
    Suppose $y$ is an $\epsilon'$-SOSP for a given $x$ (as defined in Def.~\ref{def:stationary})  and all saddle points are strict.
    For any positive $\epsilon, \mu=\sqrt{\rho\epsilon'} + \sigma$, $(x,y)$ is said to be an $(\epsilon,\mu)$-stationary solution to Eq.~\ref{eq:surrogate} if it satisfies: $\Pi(x,y) := \|\nabla \Phi_{\mu}(x)\|^2 + \epsilon' \le \epsilon$.
\end{definition}

With the formal definition of this convergence metric and the SOSP-based lower-level surrogate for the LLNC-BLO problem, the following key question naturally arises: 

\begin{tcolorbox}[left=1.2pt,right=1.2pt,top=1.2pt,bottom=1.2pt]
\textbf{Question}: Is the $\mu$-augmented hypergradient a good estimation of the ``steepest descent update direction'' of the original LLNC-BLO problem in terms of $x$? More importantly, how to control the $\mu$-parameter in the SOSP-based surrogate in order to ensure that we still converge to a good solution to the original LLNC-BLO problem with a lower-level local optimality guarantee?
\end{tcolorbox}

Although this question is challenging to answer in the highly complex LLNC setting, one may obtain some high-level qualitative understanding by measuring the ``bias'' of the $\mu$-augmented hypergradient when a hypergradient is well-defined (note that in LLNC-BLO, a hypergradient could be ill-posed if the Hessian of the lower-level problem is rank-deficient).
More specifically, note that if the sequence generated by an algorithm is able to escape strict saddle points, it typically reaches a local minimum, where the lower-level objective function usually exhibits sufficient local curvature in a bounded local neighborhood, so that its landscape can be bounded from below by a quadratic function.
Consider a scenario in which $\nabla_{yy}^2 g(\tilde{x}, y_{\mu}^*(\tilde{x})) \succeq \mu_0 I_q$ holds for some $\mu_0 > 0$, so that a hypergradient is well-defined. 
Due to the fact that $A^{-1} - B^{-1} = A^{-1}(B - A)B^{-1}$, we can bound the discrepancy between the $\mu$-augmented hypergradient and the true hypergradient as follows:
\begin{align*}
    \|\nabla\Phi_\mu(\tilde{x}) - \nabla\Phi_0(\tilde{x})\| 
    & = \left\| \nabla_{xy}^2 g \left( \left[ \nabla_{yy}^2 g \right]^{-1} (\mu I_q) \left[ \nabla_{yy}^2 g + \mu I_q \right]^{-1} \right) \nabla_y f \right\| \\
    & \le \|\nabla_{xy}^2 g\| \cdot \left\| \left[ \nabla_{yy}^2 g \right]^{-1} \right\| \cdot \mu \cdot \left\| \left[ \nabla_{yy}^2 g + \mu I_q \right]^{-1} \right\| \cdot \|\nabla_y f\| \\
    & \overset{(a)}{\le} \ell \cdot \frac{1}{\mu_0} \cdot \mu \cdot \frac{1}{\mu_0 + \mu} \cdot \nu \le \frac{\ell \nu}{\mu_0^2} (\sqrt{\rho\epsilon'} + \sigma) = \mathcal{O}(\sqrt{\epsilon'}+\sigma),
\end{align*}
where $(a)$ follows from \Cref{ass:smoothness}.
This shows that the $\mu$-augmented hypergradient incurs an approximation bias that scales at most linearly with both the square root of the SOSP violation $\sqrt{\epsilon'}$ and the augmentation parameter $\sigma$.
To mitigate the impact of such ``bias'', we can adaptively shrink both $\epsilon_t$ and $\sigma_t$ in the design of our algorithm to progressively refine the lower-level solutions.
Specifically, by using a dynamic sequence $\{\mu_t=\sqrt{\rho\epsilon_t} + \sigma_t\}_t$ that approaches $0^+$ as $t\to\infty$, the ``bias'' can be gradually eliminated to induce convergence to a solution with vanishing bias to the original LLNC-BLO problem.
Based on this intuition, we propose \alg (\underline{P}erturbed g\underline{r}adient alg\underline{o}rithm for \underline{b}ilevel probl\underline{e}m), which is summarized in Algorithms~\ref{alg:main} and \ref{alg:PGD}.
In what follows, we will first describe the basic idea and the key steps of \algn.
The theoretical convergence analysis of \alg will be provided in \Cref{sec:analysis}.

\begin{algorithm}[t!]
    \caption{The overall \alg algorithmic framework.}\label{alg:main}
    \DontPrintSemicolon
    \textbf{Input:} Iteration $T, K_t, N_t$, initial points $x_0, y_0, v_0$, step-sizes $\alpha_t,\beta_t$, aggregation parameters $\mu_t$, and PGD parameters $r_t, \varrho_t, \Delta_{G,t}, \Delta_{K,t}$, where $t\in\{0, \dotsc, T-1\}$\;
    \For{$t = 0, \dots, T-1$}{
        Update $y_{t+1} \leftarrow \text{PGD}\big( K_t, x_t, y_t, \beta_t, r_t, \varrho_t, \Delta_{G,t}, \Delta_{K,t} \big)$\;
        Solve $\big( \nabla_{yy}^2 g(x_t, y_{t+1}) + \mu_t I_q \big) v = \nabla_y f(x_t, y_{t+1})$ from $v_t$ by $N_t$-step Conjugate Gradient with a warm start to get $v_{t+1}$\;
        Update $x_{t+1} \leftarrow x_t \!-\! \alpha_t\widehat{\nabla}\Phi_{\mu_t}(x_t)$, where $\widehat{\nabla}\Phi_{\mu_t}(x_t) = \nabla_x f(x_t, y_{t+1}) \!-\! \nabla_{xy}^2 g(x_t, y_{t+1}) v_{t+1}$\;
    }
    \Return $(x_T, y_T)$\;
\end{algorithm}
\begin{algorithm}[t!]
    \caption{The perturbed gradient descent subroutine PGD($K, x, y^0, \beta, r, \varrho, \Delta_G, \Delta_K$).}\label{alg:PGD}
    \DontPrintSemicolon
    $\tau \leftarrow -\Delta_K - 1$\;
    \For{$k = 0, \dots, K-1$}{
        \If{$\|\nabla_y g(x, y^k)\| \le \varrho$ \textbf{and} $k - \tau > \Delta_K$}{
            $\tilde{y} \leftarrow y^k, \,\,\,\,\,\, \tau \leftarrow k$ \tcp*{Store anchor point \& Update timestamp}
            $\xi \sim \text{Unif}(\mathbb{B}_{\bf{0}}(r)), \,\,\,\,\,\, y^k \leftarrow y^k + \xi$ \tcp*{Inject uniform noise}
        }
        \If{$k - \tau = \Delta_K$ \textbf{and} $g(x, y^k) - g(x, \tilde{y}) > -\Delta_G$}{
            \Return $\tilde{y}$ \tcp*{Return confirmed SOSP}
        }
        $y^{k+1} \leftarrow y^k - \beta \nabla_y g(x, y^k)$\;
    }
\end{algorithm}

\textbf{1) Basic Idea:} \alg employs a double-loop framework,
where the inner loop is dedicated to finding an $\epsilon_t$-SOSP of the lower-level objective $g(x_t, \cdot)$ for a given upper-level variable $x_t$ in step $t$, while
the outer loop utilizes this approximation to construct a $\mu_t$-augmented hypergradient, which guides the update of the upper-level variable $x_t$ towards the minimum of $\Phi_{\mu_t}(\cdot)$ with a shrinking $\mu_t$.

\textbf{2) SOSP Approximation via PGD (Inner-Loop):} The inner-loop (Algorithm~\ref{alg:PGD}) employs Perturbed Gradient Descent (PGD, \cite{jin2017escape}) to identify SOSP approximations.
In order to escape saddle points, PGD monitors the gradient: if $\|\nabla_y g(x, y^k)\| \le \varrho$, it injects a uniform noise $\xi \sim \text{Unif}(\mathbb{B}_{\bf{0}}(r))$ for a more aggressive exploration, where $r$ denotes the radius of a ball and will be selected later.
After several steps, the algorithm confirms that the iterate has reached an approximate SOSP if the function value has an insufficient reduction (i.e., $g(x, y^k) - g(x, \tilde{y}) > -\Delta_G$).
Finally, the inner-loop returns the identified $\epsilon_t$-SOSP to the outer-loop at each outer iteration $t$ by adaptively and properly selecting parameters ($K_t, \varrho_t, \Delta_{G,t}, \Delta_{K,t}$).

\textbf{3) $\mu_t$-Hypergradient Estimation and Update (Outer-Loop):} Building upon the lower-level solution $y_{t+1}$ identified by the inner-loop, the outer-loop (\alg, Algorithm~\ref{alg:main}) proceeds to update the upper-level variable.
In the $\mu_t$-augmented hypergradient, the primary computational bottleneck lies in estimating the Hessian inverse $[\nabla_{yy}^2 g(x_t, y_{t+1}) + \mu_t I_q]^{-1}$.
We address this by employing the Conjugate Gradient (CG) method \citep{grazzi2020iteration,ji2021bilevel}, an inversion-free approach that only queries Hessian-vector products.
This allows \alg to efficiently compute the adjoint variable $v_{t+1}$ within $N_t$ iterations at every time slot $t$.
Furthermore, we equip the $N_t$-step CG with a warm start mechanism (conducting CG in the $t$-round beginning from the last output $v_t$), which significantly accelerates convergence to reach a faster rate.
Finally, using the estimated $\mu_t$-augmented hypergradient $\widehat{\nabla}\Phi_{\mu_t}(x_t)$, we perform a gradient descent step to update upper-level variable to $x_{t+1}$.
As will be mentioned in the next section, to ensure the convergence to a good solution of the original LLNC-BLO problem, we dynamically increase and decrease the parameters $K_t$ and $\sigma_t$, respectively, to shrink $\mu_t = \sqrt{\rho \epsilon_t} + \sigma_t$ toward $0$ at an appropriate rate.

\section{Finite-Time Convergence Analysis of \alg}\label{sec:analysis}

In this section, we will conduct theoretical convergence analysis of our proposed \alg algorithm.
Toward this end, we first restate the iteration complexity performance guarantee of the PGD method~\citep{jin2017escape}, which is adapted in the context of LLNC-BLO in Algorithm~\ref{alg:PGD}.
\begin{lemma}[Iteration Complexity of PGD]\label{lemma:PGD}
    Under \Cref{ass:smoothness}, with hyperparameter choices in Appendix~\ref{sec:app-proof}, for any $\delta'\in(0,1), \epsilon_t\le\ell^2/\rho, \Delta_{g,t} \ge g(x_t,y_t)-\inf_y g(x_t,y)$, with probability at least $1-\delta'$, \Cref{alg:PGD} outputs an $\epsilon_t$-SOSP with iterations no more than $K_t = \mathcal{O}\big( \frac{\ell\Delta_{g,t}}{\epsilon_t^2}\log^4\big(\frac{q\ell\Delta_{g,t}}{\epsilon_t^2\delta'}\big) \big)$.
\end{lemma}

\vspace{-.1in}
We note that, besides PGD, several algorithms have been proposed to escape saddle points while maintaining nearly dimension-free convergence rates \citep{jin2018accelerated, jin2021nonconvex, zhang2021escape}.
While these approaches offer improvements in iteration complexity to some degree (from $\epsilon_t^{-2}$ to $\epsilon_t^{-1.75}$ at most), they typically introduce additional algorithmic modules, such as Negative Curvature Finding (NCF) and Nesterov-style acceleration, which significantly complicate the implementation of the inner-loop algorithm.
Moreover, rather than offering a ``last-iteration'' convergence guarantee, these accelerated methods only guarantee an approximate SOSP being visited at some point within the trajectory. This renders them ill-suited for algorithm designs for LLNC-BLO, which requires a terminal SOSP estimate.
We are now ready to state the main convergence result for \alg.
\begin{theorem}\label{thm:main}
    By selecting $K_t = \widetilde{\Theta}\big((t+1)^{\frac{8}{5}}\big)$, $\sigma_t=\Theta\big((t+1)^{-p}\big)$ and $\alpha_t=\Theta\big((t+1)^{-q}\big)$ with $p\in(0,\frac{1}{3}), 3p \le q < 1$,
    for any $\delta\in(0,1)$, with probability at least $1-\delta$, the sequence $\{x_t,y_t\}_t$ generated by \Cref{alg:main} satisfies the following convergence guarantee: $\min_{\frac{T}{2}\le t\le T} \Pi(x_t, y_t) = \mathcal{O}( \frac{1}{T^{1-q}} + \frac{1}{T^{2p}} + \frac{1}{T^{\frac{4}{5}}} )$.
    By selecting $(p,q)=(\frac{1}{5},\frac{3}{5})$, we have $\min_{\frac{T}{2}\le t\le T} \Pi(x_t, y_t) = \mathcal{O}\big( T^{-\frac{2}{5}} \big)$.
\end{theorem}
\textit{Proof Sketch.} Our convergence analysis is organized in three key steps (proof details of \Cref{thm:main} are relegated to Appendix~\ref{sec:app-proof} due to limited space):
\vspace{-0.5em}
\begin{list}{\labelitemi}{\leftmargin=1.5em \itemindent=-0.0em \itemsep=-.2em}
    \item \textbf{Step 1) Bounding the Variable Drifts:} We first control the tracking errors of the lower-level variable $y$ and the adjoint multiplier $v$.
    A major challenge here is decomposing and bounding several intertwined gaps: the inner-iteration solver error (i.e., taking $y$ for example hereafter, $\|y_{t+1} - y_{\mu_t}^*(x_t) \|$), and the parameter-induced shift across varying $\mu$ (i.e., $\|y_{\mu_{t+1}}^*(x_t) - y_{\mu_t}^*(x_t)\|$).
    \item \textbf{Step 2) Controlling the Augmented Implicit Function Shift:} Next, we analyze the dynamics of $\Phi_{\mu_{t+1}}(x_{t+1}) - \Phi_{\mu_t}(x_t)$, which is challenging due to simultaneous updates in both the variable $x$ and the parameter $\mu$.
    We decompose this shift into two components: (i) the contribution from the dynamics of $x$, which can be controlled using the smoothness of the augmented implicit function and, consequently, the descent lemma; and (ii) the contribution induced by the drift in $\mu$, which can be converted to the drift in $y$ that has already been controlled in the previous step.
    \item \textbf{Step 3) Lyapunov Analysis and Telescoping:} Finally, building upon the aforementioned results, we construct a dynamic Lyapunov function to decouple the descent of $\Phi_{\mu_t}$ from the tracking errors.
    By using a telescoping sum, we show that all error terms are bounded by the sum of a constant and a convergent series, thereby establishing the finite-time convergence rate.\qed
\end{list}

Three remarks on \Cref{thm:main} are in order:
{\bf First}, to our knowledge, this work is the \textit{first} that incorporates second-order stationarity as a surrogate for LLNC-BLO and propose an algorithm with a finite-time convergence rate of $\mathcal{O}\big( T^{-\frac{2}{5}} \big)$ with lower-level SOSP guarantee.
While existing LLNC-BLO methods \citep{shen2023penalty, xiao2023generalized, liu2024moreau, ma2026sun} also provided finite-time convergence analysis, a direct comparison with them is difficult due to the fact that most of these works assume the PL condition, which is not needed in this work.
{\bf Second}, there is a trade-off in choosing $\sigma_t = \sigma_0 (t+1)^{-p}$:
a larger $p$-value yields faster convergence rate $\mathcal{O}(T^{-2p})$ in \Cref{thm:main}.
However, the condition $q\ge3p$ implies a small learning rate $\alpha_t = \alpha_0 (t+1)^{-q}$.
Also, the smoothness coefficient of $\Phi_{\mu_t}$ scales as $L_t = \Theta(\sigma_t^{-3})$, implying that $\alpha_t = \Theta(L_t^{-1}) = \Theta((t+1)^{-3p})$ when $q=3p$ is required for a fast convergence, which in turn necessitates a relatively small $p$.
To balance these two directions, we choose $(p,q)=(\frac{1}{5}, \frac{3}{5})$ to achieve the most efficient convergence rate $\mathcal{O}\big( T^{-\frac{2}{5}} \big)$.
{\bf Third}, instead of using a fixed number of inner-loop iterations $K$, \alg gradually increases $K_t$ over time, leading to a \textit{two-timescale} algorithmic design.
Specifically,
we introduce a $\mu$-augmentation to allow a well-defined hypergradient in LLNC-BLO.
This term introduces an error that can be controlled by gradually letting $\{\mu_t\}\to0$.
Hence, $K_t$ needs to be increasing rather than fixed to decrease the inner-loop error $\epsilon_t$ and thus $\mu_t$, leading to this two-timescale design.

\section{Experimental Results}

\begin{wrapfigure}{r}{0.6\textwidth}
    \vspace{-0.1in}
    \centering
    \begin{subfigure}[t]{0.485\linewidth}
        \centering
        \includegraphics[width=\textwidth]{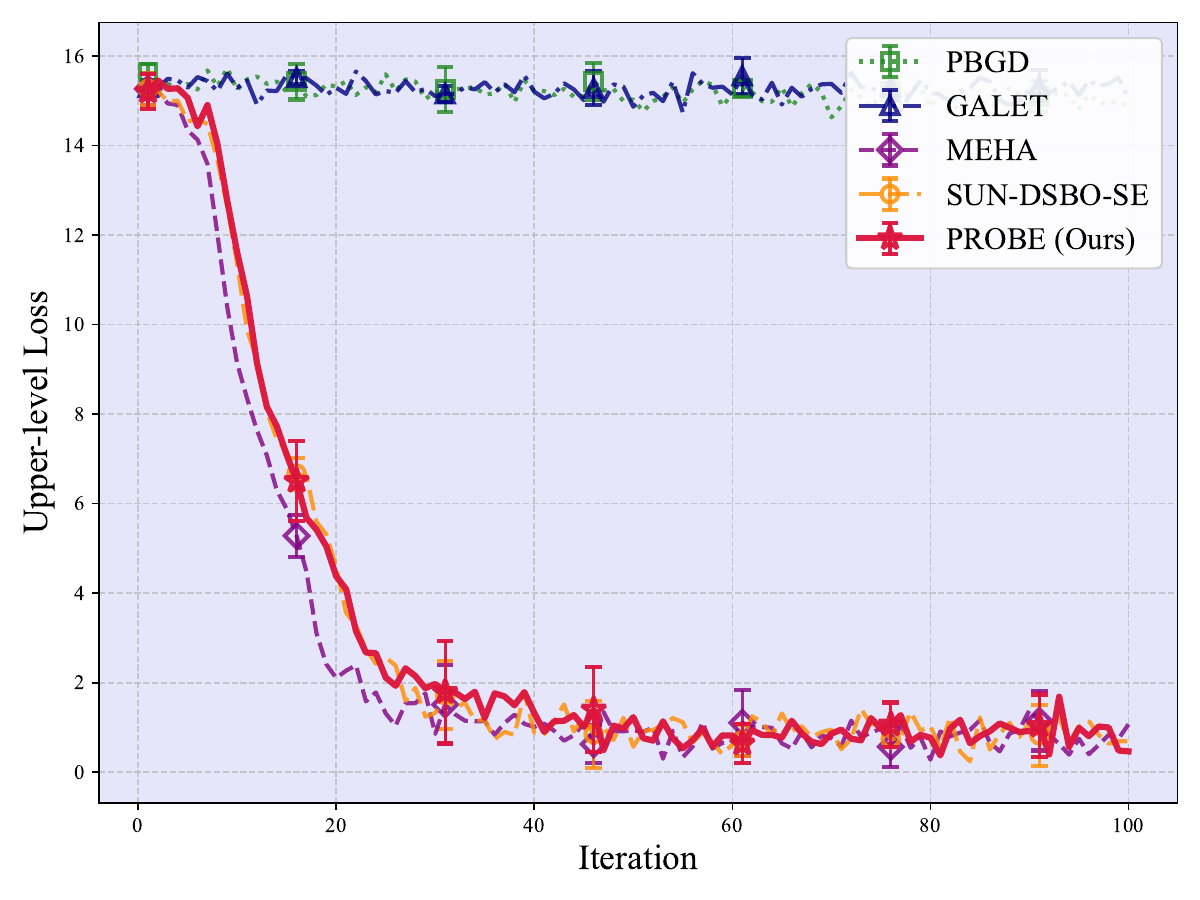}
        \vspace{-1.5em}
        \caption{Loss vs Iteration.}
        \label{fig:baseline-ul-iteration}
    \end{subfigure}
    \hfill
    \begin{subfigure}[t]{0.485\linewidth}
        \centering
        \includegraphics[width=\textwidth]{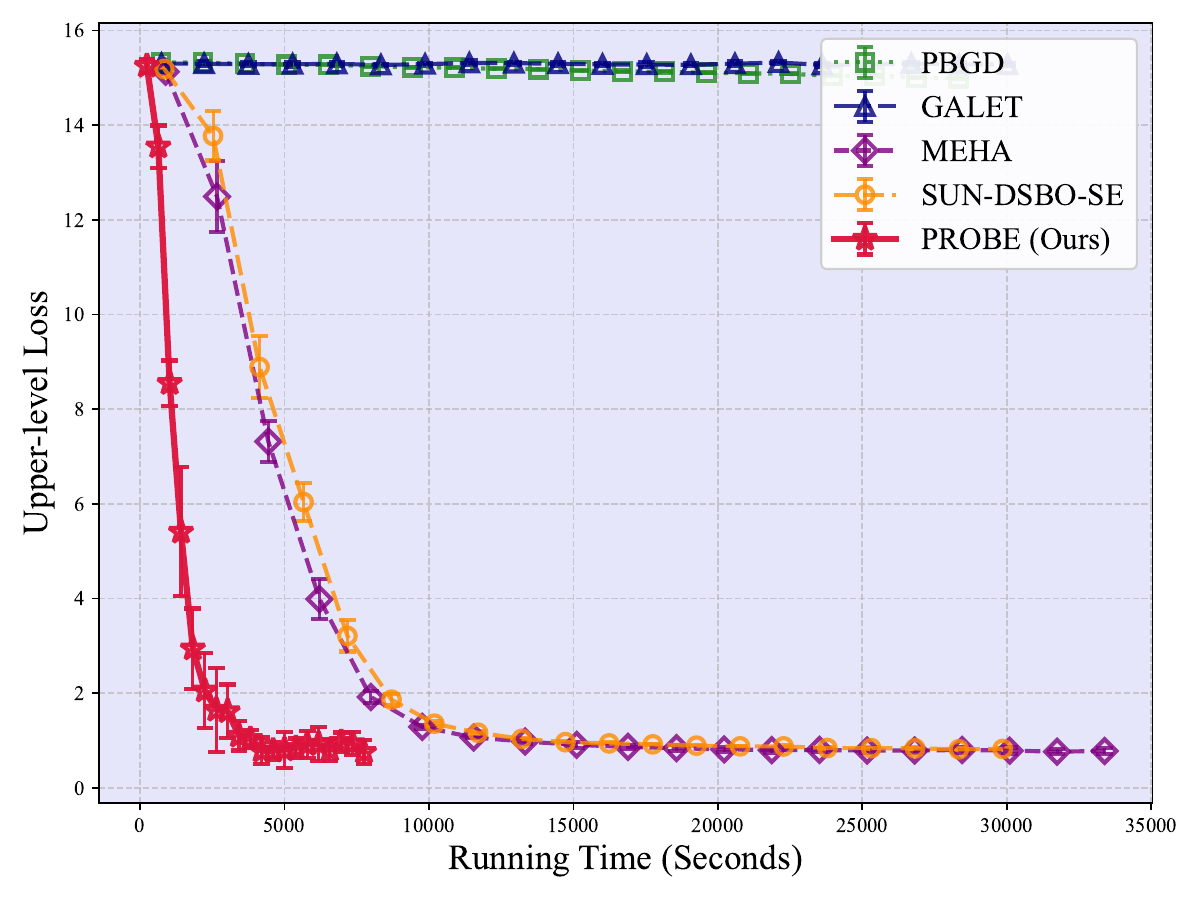}
        \vspace{-1.5em}
        \caption{Loss vs Running time.}
        \label{fig:baseline-ul-time}
    \end{subfigure}
    \vspace{-0.5em}
    \caption{Baseline comparison.}\label{fig:baseline}\vspace{-.15in}
\end{wrapfigure}

In this section, we will validate the efficacy of our \alg method by using a data curation task for LLM finetuning.
Specifically, we first curate a high quality dataset from multiple datasets of unknown qualities,
which is then used to fine-tune an LLM to mitigate toxicity, reduce verbosity, and improve coherence.
This data curation task can be formulated as an LLNC-BLO problem \citep{shen2024seal,pan2025scalebio}, where the lower level aims to determine the mixing proportion of each dataset, and the upper level further fine-tunes the model as a validation process:
\begin{align*}
    \min_{x,y} \sum\nolimits_{i=1}^{n_0} \mathcal{L}\big(r_{y(x)}(p_i), r_i\big), \,\,\,\,\,\, \text{s.t. } y(x)\in \argmin_y \sum\nolimits_{j=1}^\mathcal{D} \text{SoftMax}_x(j)\sum\nolimits_{i=1}^{n_j} \mathcal{L}\big(r_{y}(p_i), r_i\big),
\end{align*}
where $x$ is the weight vector for training datasets, $y$ represents the LLM parameters, $(p,r)$ denotes a prompt-response pair, $\mathcal{L}$ is the loss, $\text{SoftMax}_x(j)=\exp(x_j)/\sum_{j'}\exp(x_{j'})$, index $0$ stands for validation set, and $1, \dotsc, \mathcal{D}$ are training indices.
Clearly, the lower-level problem is nonconvex.
We conduct experiments on the HelpSteer \citep{wang2024helpsteer} dataset, which incorporates data of varying quality.
We employ Llama-3.2-3B-Instruct \citep{meta2024} as the base model.
More implementation details are relegated to Appendix~\ref{sec:app-exp} due to space limitation.

\textbf{1) Baseline Comparison:}
We compare our \alg method against state-of-the-art LLNC-BLO baselines: PBGD \citep{shen2023penalty}, GALET \citep{xiao2023generalized}, MEHA \citep{liu2024moreau}, and SUN-DSBO-SE \citep{ma2026sun}.
\Cref{fig:baseline} illustrates the upper-level loss along with the standard error bars over $5$ trials, which shows that PBGD and GALET fail to converge on this data curation task.
As detailed in Appendix~\ref{sec:app-exp}, PBGD performs well at the lower level.
However, its oversimplified FOSP-based surrogate could not avoid being trapped at undesirable saddle points at the lower level, leading to weak performance at the upper level.
In contrast, MEHA, SUN-DSBO-SE, and \alg converge to much better solutions at upper level.
Although the theoretical convergence rates of these methods are not directly comparable due to different metrics, \alg outperforms all baselines in terms of wall-clock running time, demonstrating its efficiency.

\begin{wrapfigure}{r}{0.6\textwidth}
    \vspace{-0.125in}
    \centering
    \begin{subfigure}[t]{0.485\linewidth}
        \centering
        \includegraphics[width=\textwidth]{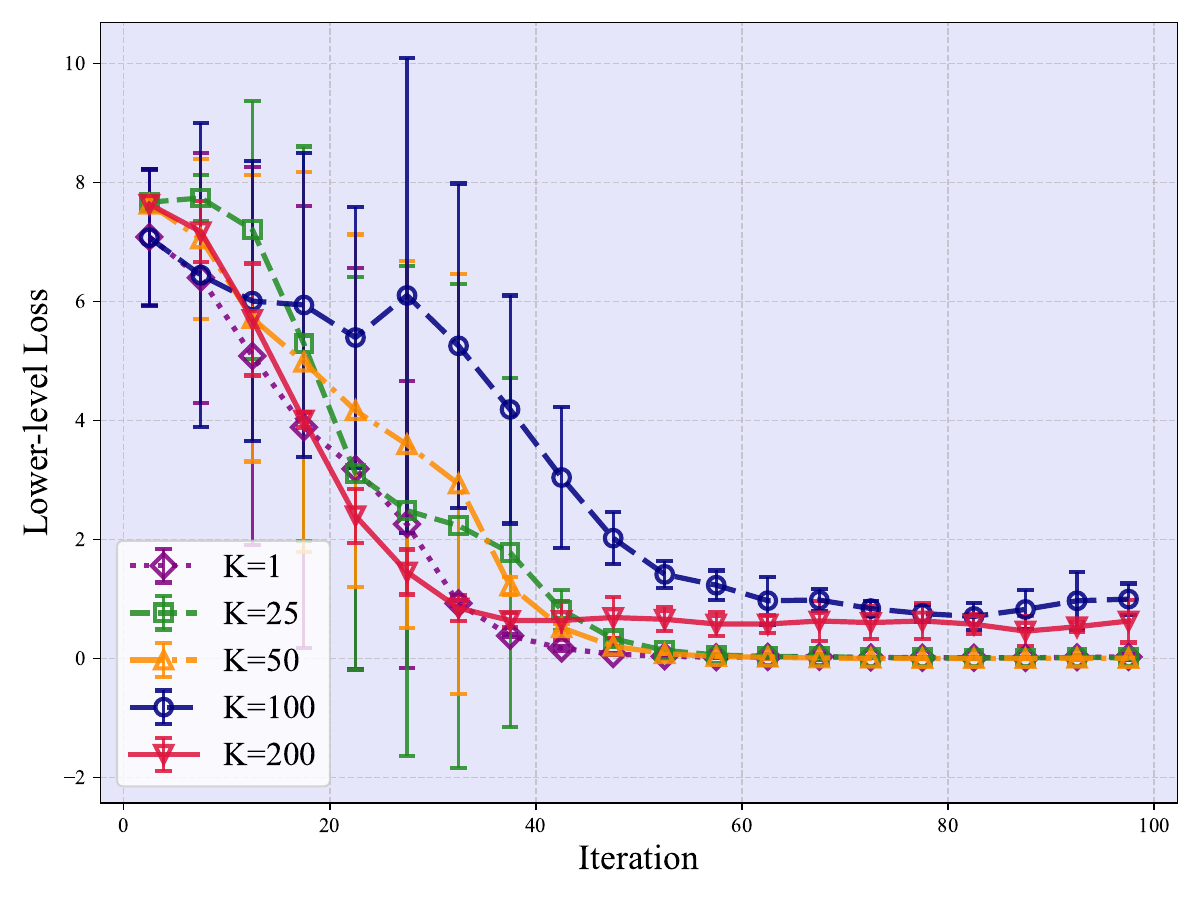}
        \vspace{-1.5em}
        \caption{Lower-level loss with $K$.}
        \label{fig:hyper-1}
    \end{subfigure}
    \hfill
    \begin{subfigure}[t]{0.485\linewidth}
        \centering
        \includegraphics[width=\textwidth]{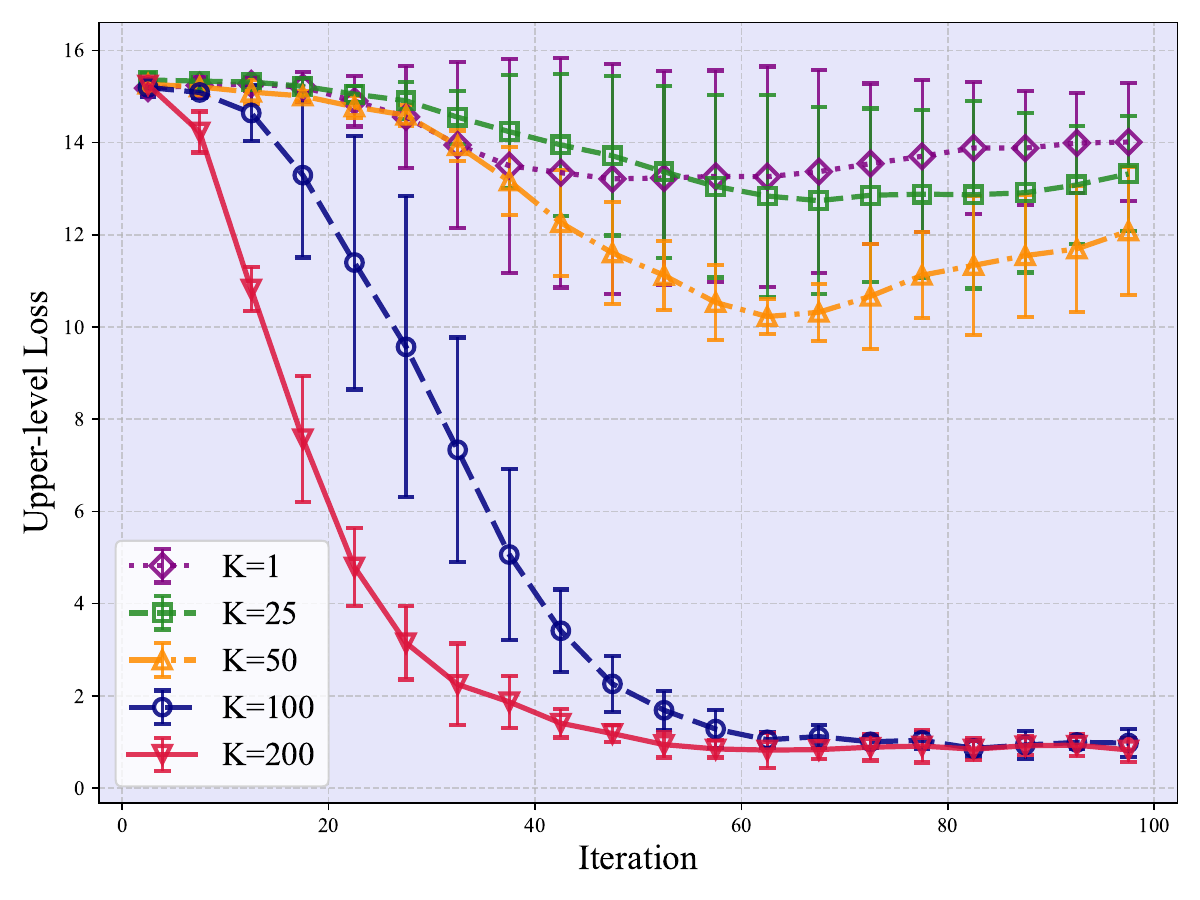}
        \vspace{-1.5em}
        \caption{Upper-level loss with $K$.}
        \label{fig:hyper-2}
    \end{subfigure}
    \begin{subfigure}[t]{0.485\linewidth}
        \centering
        \includegraphics[width=\textwidth]{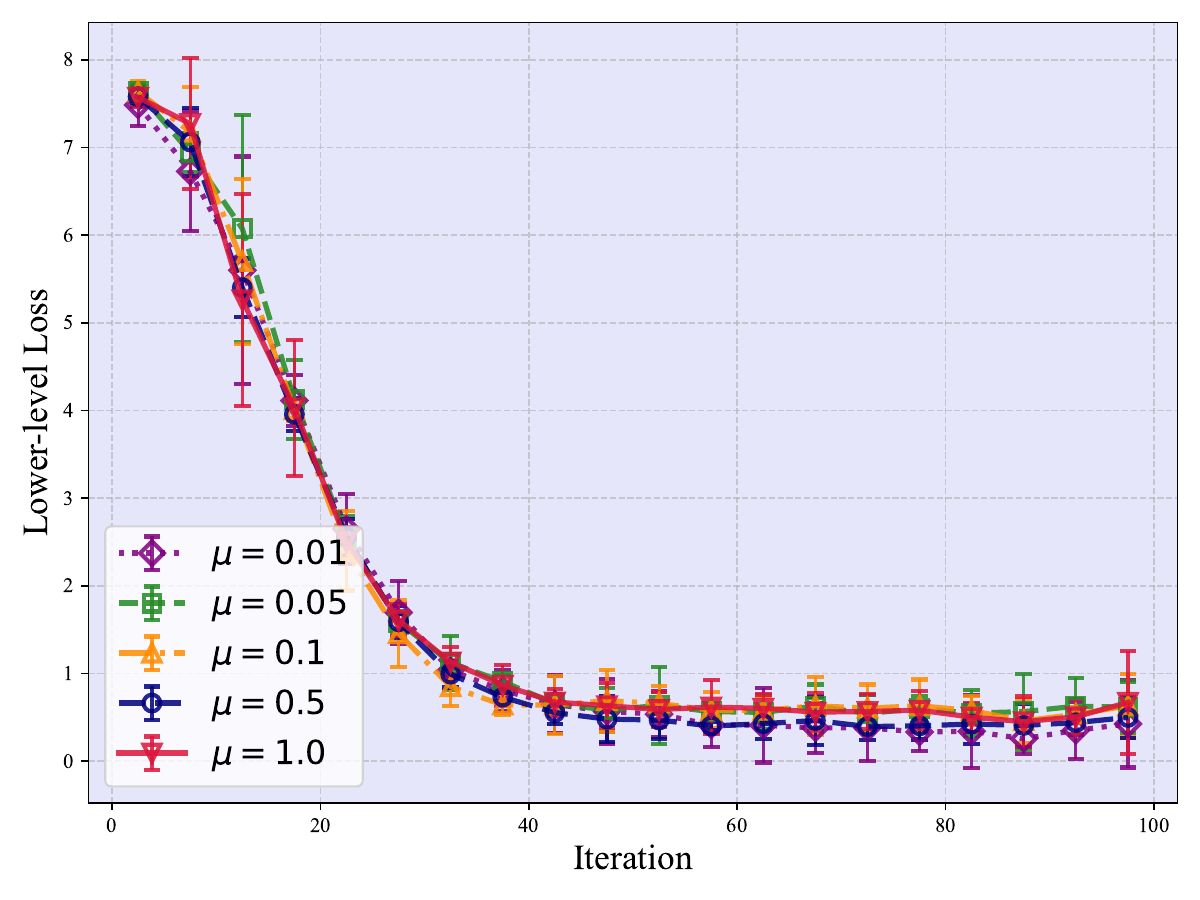}
        \vspace{-1.5em}
        \caption{Lower-level loss with $\mu$.}
        \label{fig:hyper-3}
    \end{subfigure}
    \hfill
    \begin{subfigure}[t]{0.485\linewidth}
        \centering
        \includegraphics[width=\textwidth]{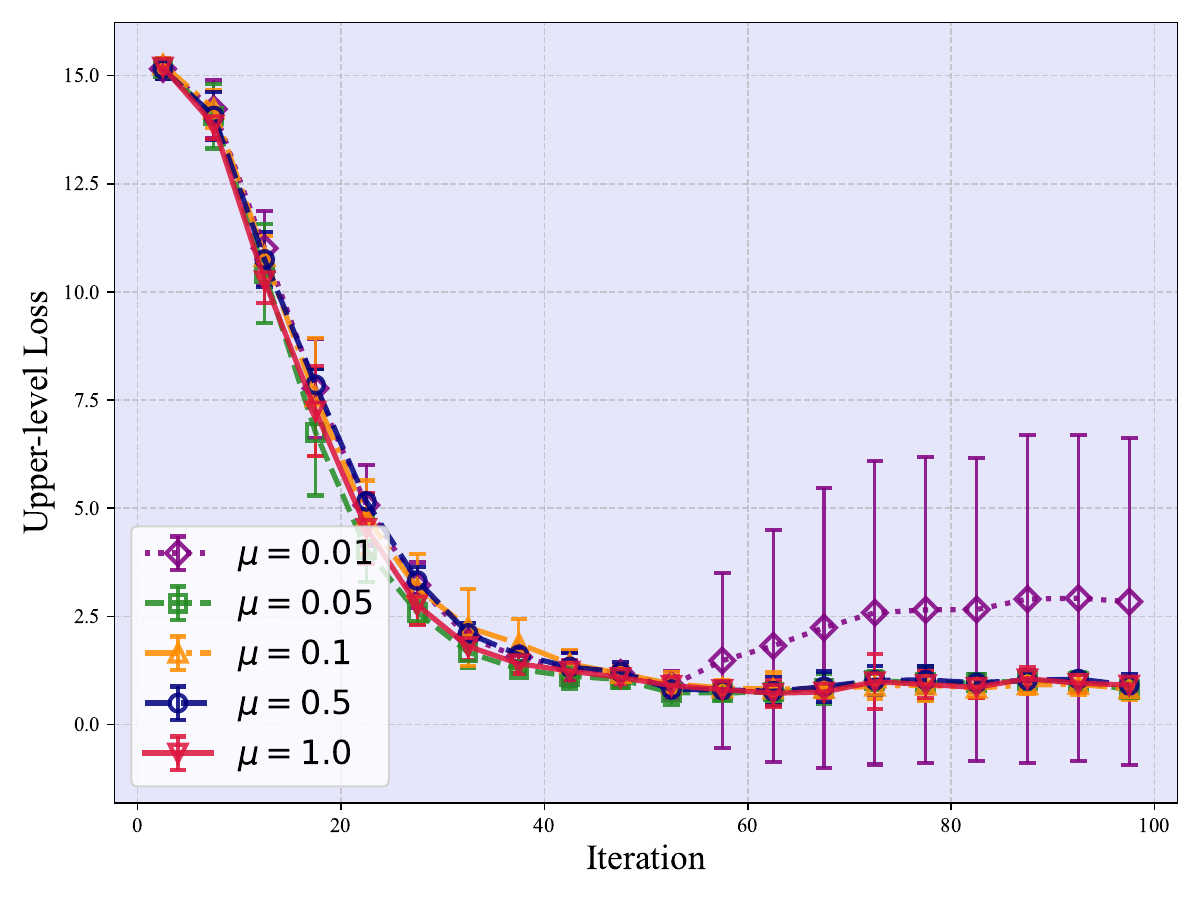}
        \vspace{-1.5em}
        \caption{Upper-level loss with $\mu$.}
        \label{fig:hyper-4}
    \end{subfigure}
    \vspace{-0.5em}
    \caption{Hyperparameter investigation.}
    \label{fig:HYPER}
    \vspace{-.125in}
\end{wrapfigure}

\textbf{2) Hyper-parameter Experimentation and Ablation Study:}
In this experiment, we study how the selections of hyperparameters $K$ and $\mu$ impact the performance of \algn, as shown in Fig.~\ref{fig:HYPER}.
For simplicity, we fix $K$ and $\mu$ for each experiment.
As shown in \Cref{fig:hyper-1,fig:hyper-2}, increasing the inner-loop PGD iterations, $K$, yields a significant improvement in the upper-level loss.
This aligns with our theoretical insights: a small $K$ restricts PGD's ability to probe and escape saddle points.
Consequently, the lower-level problem stagnates at a suboptimal position, which in turn leads to unsatisfactory upper-level performance.
In contrast, when $K$ is sufficiently large, \alg reliably finds a lower-level SOSP, yielding substantial overall performance improvements.
\Cref{fig:hyper-3,fig:hyper-4} show the influence of the augmentation parameter $\mu$, both of which show that the larger $\mu$ facilitates a slightly lower loss value. 
Also, \Cref{fig:hyper-4} shows that the trajectory fluctuates dramatically when $\mu$ is insufficient to restore the Hessian positive definiteness.
This highlights the important tradeoff and provides a guidance in selecting $\mu$.

\begin{wrapfigure}{r}{0.6\textwidth}
    \vspace{-0.15in}
    \centering
    \begin{subfigure}[t]{0.485\linewidth}
        \centering
        \includegraphics[width=\textwidth]{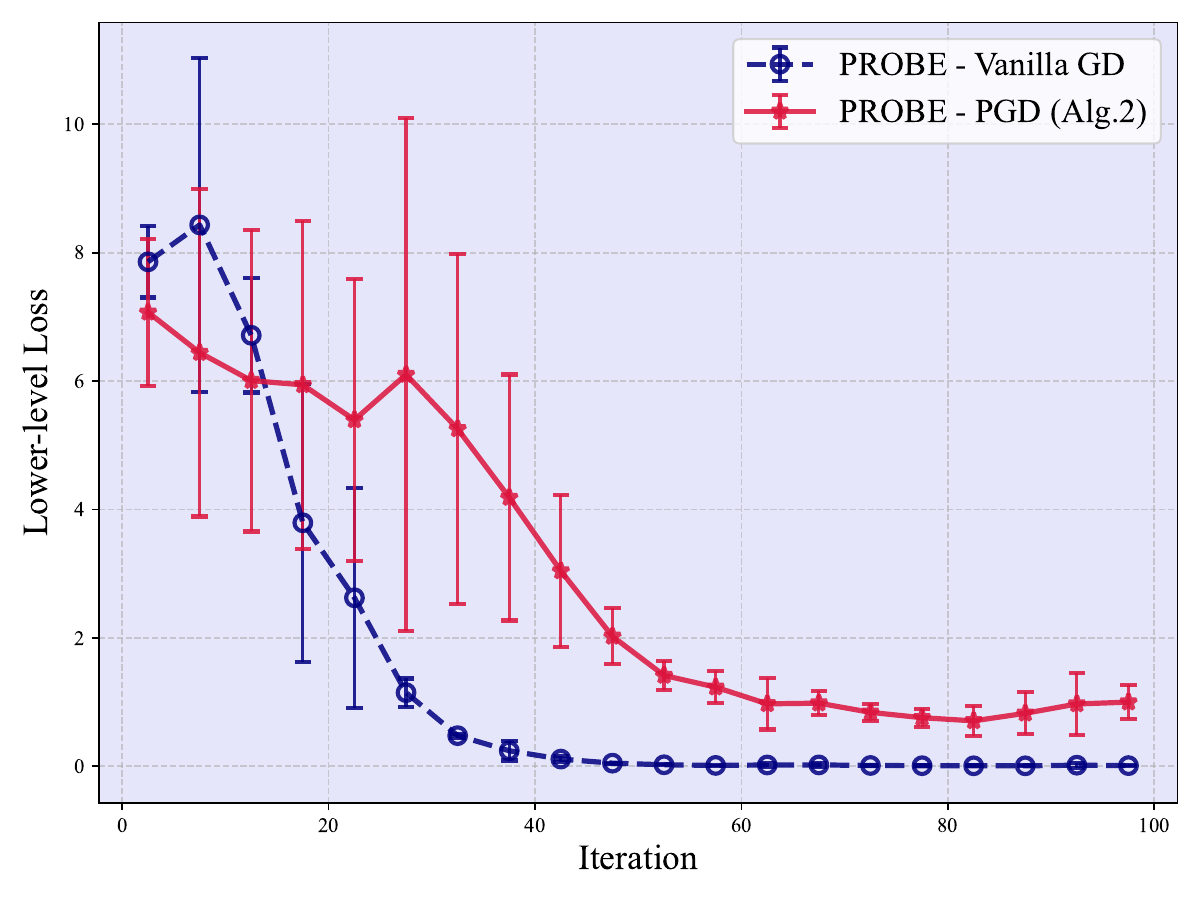}
        \vspace{-1.5em}
        \caption{Lower-level loss.}
        \label{fig:ablation-ll}
    \end{subfigure}
    \hfill
    \begin{subfigure}[t]{0.485\linewidth}
        \centering
        \includegraphics[width=\textwidth]{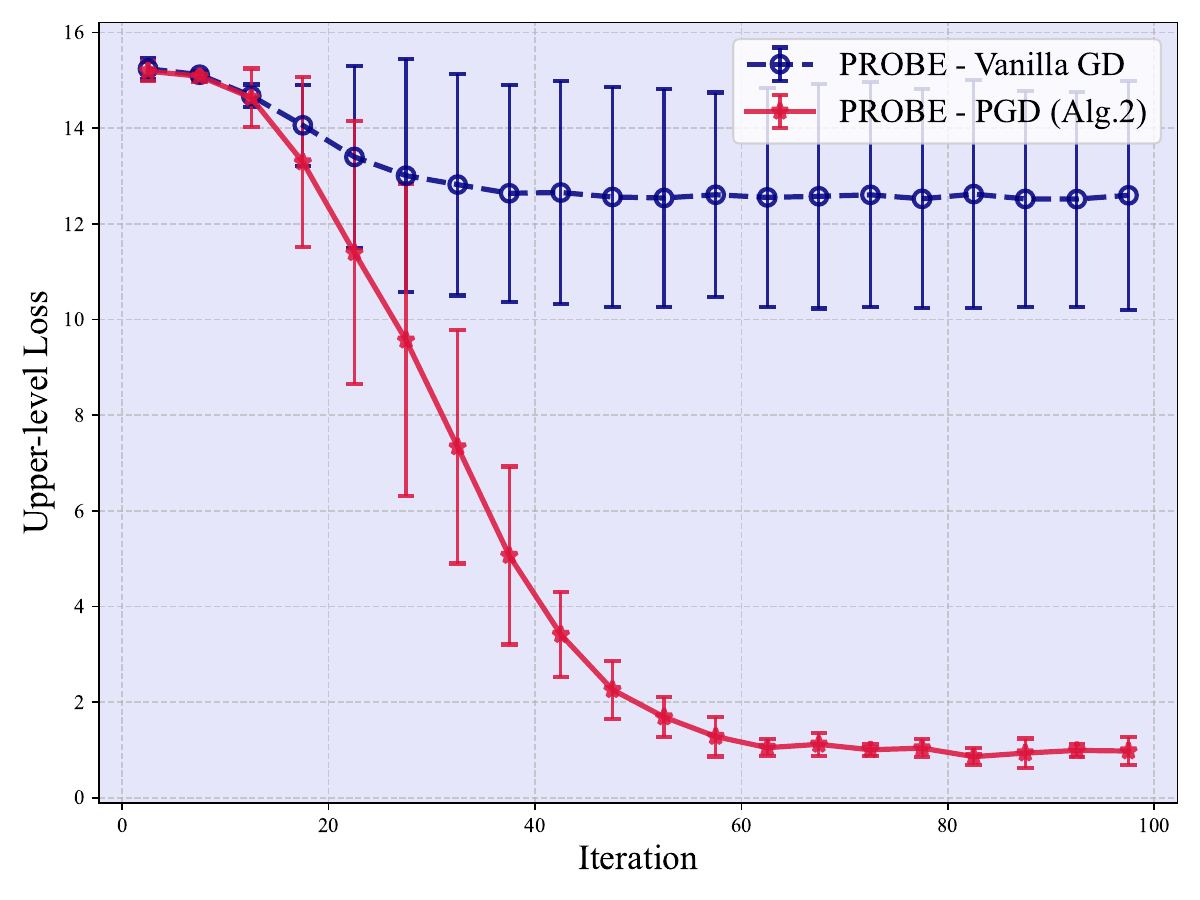}
        \vspace{-1.5em}
        \caption{Upper-level loss.}
        \label{fig:ablation-ul}
    \end{subfigure}
    \vspace{-0.5em}
    \caption{Ablation Study.}\label{fig:ablation}\vspace{-.125in}
\end{wrapfigure}
We also conduct the following ablation study by replacing the PGD method in Algorithm~\ref{alg:PGD} with the vanilla gradient descent method, keeping all other components of \alg unchanged.
As shown in \Cref{fig:ablation}, without saddle points escaping, \alg is easily trapped in solutions seemingly better at the lower level but are suboptimal at the upper level, highlighting the importance of our SOSP-based lower-level surrogate and the efficacy of our proposed \alg.
Due to space limitation, we provide further numerical details and a meta-learning experiment in Appendix~\ref{sec:app-exp}.

\section{Conclusion}

In this paper, we investigated bilevel optimization with non-convex lower levels.
We showed that existing first-order stationarity-based surrogates often fail to escape lower-level saddle points.
To address this challenge, we proposed a second-order stationarity-based surrogate, based on which we developed a new algorithm called \alg and established its finite-time convergence rate.
We conducted experiments on dataset curation tasks for LLM finetuning to validate our \alg algorithm.
We note that the focus of this work is on deterministic LLNC-BLO problems.
Future directions include extending our framework into stochastic LLNC-BLO problems.

\bibliography{reference}

@inproceedings{ge2015escaping,
  title={Escaping from saddle points—online stochastic gradient for tensor decomposition},
  author={Ge, Rong and Huang, Furong and Jin, Chi and Yuan, Yang},
  booktitle={Conference on learning theory},
  pages={797--842},
  year={2015},
  organization={PMLR}
}

@inproceedings{jin2017escape,
  title={How to escape saddle points efficiently},
  author={Jin, Chi and Ge, Rong and Netrapalli, Praneeth and Kakade, Sham M and Jordan, Michael I},
  booktitle={International conference on machine learning},
  pages={1724--1732},
  year={2017},
  organization={PMLR}
}

@inproceedings{jin2018accelerated,
  title={Accelerated gradient descent escapes saddle points faster than gradient descent},
  author={Jin, Chi and Netrapalli, Praneeth and Jordan, Michael I},
  booktitle={Conference on learning theory},
  pages={1042--1085},
  year={2018},
  organization={PMLR}
}

@article{zhang2021escape,
  title={Escape saddle points by a simple gradient-descent based algorithm},
  author={Zhang, Chenyi and Li, Tongyang},
  journal={Advances in Neural Information Processing Systems},
  volume={34},
  pages={8545--8556},
  year={2021}
}

@article{xiao2023generalized,
  title={A generalized alternating method for bilevel learning under the polyak-$\{$$\backslash$L$\}$ ojasiewicz condition},
  author={Xiao, Quan and Lu, Songtao and Chen, Tianyi},
  journal={arXiv preprint arXiv:2306.02422},
  year={2023}
}

@inproceedings{shen2023penalty,
  title={On penalty-based bilevel gradient descent method},
  author={Shen, Han and Chen, Tianyi},
  booktitle={International conference on machine learning},
  pages={30992--31015},
  year={2023},
  organization={PMLR}
}

@article{ma2026sun,
  title={SUN-DSBO: A Structured Unified Framework for Nonconvex Decentralized Stochastic Bilevel Optimization},
  author={Ma, Yaoshuai and Wang, Xiao and Yao, Wei and Zhang, Jin},
  journal={arXiv preprint arXiv:2601.22682},
  year={2026}
}

@article{liu2024moreau,
  title={Moreau envelope for nonconvex bi-level optimization: A single-loop and hessian-free solution strategy},
  author={Liu, Risheng and Liu, Zhu and Yao, Wei and Zeng, Shangzhi and Zhang, Jin},
  journal={arXiv preprint arXiv:2405.09927},
  year={2024}
}

@inproceedings{ji2021bilevel,
  title={Bilevel optimization: Convergence analysis and enhanced design},
  author={Ji, Kaiyi and Yang, Junjie and Liang, Yingbin},
  booktitle={International conference on machine learning},
  pages={4882--4892},
  year={2021},
  organization={PMLR}
}

@article{bracken1973mathematical,
  title={Mathematical programs with optimization problems in the constraints},
  author={Bracken, Jerome and McGill, James T},
  journal={Operations research},
  volume={21},
  number={1},
  pages={37--44},
  year={1973},
  publisher={INFORMS}
}

@article{zhang2024introduction,
  title={An introduction to bilevel optimization: Foundations and applications in signal processing and machine learning},
  author={Zhang, Yihua and Khanduri, Prashant and Tsaknakis, Ioannis and Yao, Yuguang and Hong, Mingyi and Liu, Sijia},
  journal={IEEE Signal Processing Magazine},
  volume={41},
  number={1},
  pages={38--59},
  year={2024},
  publisher={IEEE}
}

@article{liu2021investigating,
  title={Investigating bi-level optimization for learning and vision from a unified perspective: A survey and beyond},
  author={Liu, Risheng and Gao, Jiaxin and Zhang, Jin and Meng, Deyu and Lin, Zhouchen},
  journal={IEEE Transactions on Pattern Analysis and Machine Intelligence},
  volume={44},
  number={12},
  pages={10045--10067},
  year={2021},
  publisher={IEEE}
}

@article{hong2023two,
  title={A two-timescale stochastic algorithm framework for bilevel optimization: Complexity analysis and application to actor-critic},
  author={Hong, Mingyi and Wai, Hoi-To and Wang, Zhaoran and Yang, Zhuoran},
  journal={SIAM Journal on Optimization},
  volume={33},
  number={1},
  pages={147--180},
  year={2023},
  publisher={SIAM}
}

@article{kudo2026sample,
  title={Sample-Efficient Hypergradient Estimation for Decentralized Bi-Level Reinforcement Learning},
  author={Kudo, Mikoto and Tanabe, Takumi and Wachi, Akifumi and Akimoto, Youhei},
  journal={arXiv preprint arXiv:2603.14867},
  year={2026}
}

@inproceedings{franceschi2018bilevel,
  title={Bilevel programming for hyperparameter optimization and meta-learning},
  author={Franceschi, Luca and Frasconi, Paolo and Salzo, Saverio and Grazzi, Riccardo and Pontil, Massimiliano},
  booktitle={International conference on machine learning},
  pages={1568--1577},
  year={2018},
  organization={PMLR}
}

@inproceedings{zhang2022revisiting,
  title={Revisiting and advancing fast adversarial training through the lens of bi-level optimization},
  author={Zhang, Yihua and Zhang, Guanhua and Khanduri, Prashant and Hong, Mingyi and Chang, Shiyu and Liu, Sijia},
  booktitle={International Conference on Machine Learning},
  pages={26693--26712},
  year={2022},
  organization={PMLR}
}

@article{shen2024seal,
  title={Seal: Safety-enhanced aligned llm fine-tuning via bilevel data selection},
  author={Shen, Han and Chen, Pin-Yu and Das, Payel and Chen, Tianyi},
  journal={arXiv preprint arXiv:2410.07471},
  year={2024}
}

@article{ghadimi2018approximation,
  title={Approximation methods for bilevel programming},
  author={Ghadimi, Saeed and Wang, Mengdi},
  journal={arXiv preprint arXiv:1802.02246},
  year={2018}
}

@article{yang2021provably,
  title={Provably faster algorithms for bilevel optimization},
  author={Yang, Junjie and Ji, Kaiyi and Liang, Yingbin},
  journal={Advances in Neural Information Processing Systems},
  volume={34},
  pages={13670--13682},
  year={2021}
}

@article{dagreou2022framework,
  title={A framework for bilevel optimization that enables stochastic and global variance reduction algorithms},
  author={Dagr{\'e}ou, Mathieu and Ablin, Pierre and Vaiter, Samuel and Moreau, Thomas},
  journal={Advances in Neural Information Processing Systems},
  volume={35},
  pages={26698--26710},
  year={2022}
}

@inproceedings{liu2023averaged,
  title={Averaged method of multipliers for bi-level optimization without lower-level strong convexity},
  author={Liu, Risheng and Liu, Yaohua and Yao, Wei and Zeng, Shangzhi and Zhang, Jin},
  booktitle={International Conference on Machine Learning},
  pages={21839--21866},
  year={2023},
  organization={PMLR}
}

@article{cao2023projection,
  title={Projection-free methods for stochastic simple bilevel optimization with convex lower-level problem},
  author={Cao, Jincheng and Jiang, Ruichen and Abolfazli, Nazanin and Yazdandoost Hamedani, Erfan and Mokhtari, Aryan},
  journal={Advances in Neural Information Processing Systems},
  volume={36},
  pages={6105--6131},
  year={2023}
}

@inproceedings{jiang2023conditional,
  title={A conditional gradient-based method for simple bilevel optimization with convex lower-level problem},
  author={Jiang, Ruichen and Abolfazli, Nazanin and Mokhtari, Aryan and Hamedani, Erfan Yazdandoost},
  booktitle={International Conference on Artificial Intelligence and Statistics},
  pages={10305--10323},
  year={2023},
  organization={PMLR}
}

@article{kwon2023penalty,
  title={On penalty methods for nonconvex bilevel optimization and first-order stochastic approximation},
  author={Kwon, Jeongyeol and Kwon, Dohyun and Wright, Stephen and Nowak, Robert},
  journal={arXiv preprint arXiv:2309.01753},
  year={2023}
}

@article{jiang2025discretization,
  title={A Discretization Approach for Bilevel Optimization with Low-Dimensional and Non-Convex Lower-Level},
  author={Jiang, Xiaotian and Tsaknakis, Ioannis and Khanduri, Prashant and Hong, Mingyi},
  journal={arXiv preprint arXiv:2505.10830},
  year={2025}
}

@article{jiang2025correspondence,
  title={A correspondence-driven approach for bilevel decision-making with nonconvex lower-level problems},
  author={Jiang, Xiaotian and Li, Jiaxiang and Bi, Jiawen and Hong, Mingyi and Zhang, Shuzhong},
  journal={arXiv preprint arXiv:2509.01148},
  year={2025}
}

@inproceedings{chen2024finding,
  title={On finding small hyper-gradients in bilevel optimization: Hardness results and improved analysis},
  author={Chen, Lesi and Xu, Jing and Zhang, Jingzhao},
  booktitle={The Thirty Seventh Annual Conference on Learning Theory},
  pages={947--980},
  year={2024},
  organization={PMLR}
}

@article{dauphin2014identifying,
  title={Identifying and attacking the saddle point problem in high-dimensional non-convex optimization},
  author={Dauphin, Yann N and Pascanu, Razvan and Gulcehre, Caglar and Cho, Kyunghyun and Ganguli, Surya and Bengio, Yoshua},
  journal={Advances in neural information processing systems},
  volume={27},
  year={2014}
}

@article{sun2018geometric,
  title={A geometric analysis of phase retrieval},
  author={Sun, Ju and Qu, Qing and Wright, John},
  journal={Foundations of Computational Mathematics},
  volume={18},
  number={5},
  pages={1131--1198},
  year={2018},
  publisher={Springer}
}

@article{nesterov2006cubic,
  title={Cubic regularization of Newton method and its global performance},
  author={Nesterov, Yurii and Polyak, Boris T},
  journal={Mathematical programming},
  volume={108},
  number={1},
  pages={177--205},
  year={2006},
  publisher={Springer}
}

@article{bhojanapalli2016global,
  title={Global optimality of local search for low rank matrix recovery},
  author={Bhojanapalli, Srinadh and Neyshabur, Behnam and Srebro, Nati},
  journal={Advances in Neural Information Processing Systems},
  volume={29},
  year={2016}
}

@article{ge2016matrix,
  title={Matrix completion has no spurious local minimum},
  author={Ge, Rong and Lee, Jason D and Ma, Tengyu},
  journal={Advances in neural information processing systems},
  volume={29},
  year={2016}
}

@inproceedings{grazzi2020iteration,
  title={On the iteration complexity of hypergradient computation},
  author={Grazzi, Riccardo and Franceschi, Luca and Pontil, Massimiliano and Salzo, Saverio},
  booktitle={International Conference on Machine Learning},
  pages={3748--3758},
  year={2020},
  organization={PMLR}
}

@article{huang2025efficiently,
  title={Efficiently escaping saddle points in bilevel optimization},
  author={Huang, Minhui and Chen, Xuxing and Ji, Kaiyi and Ma, Shiqian and Lai, Lifeng},
  journal={Journal of machine learning research},
  volume={26},
  number={1},
  pages={1--61},
  year={2025}
}

@article{ye1995optimality,
  title={Optimality conditions for bilevel programming problems},
  author={Ye, Jane J and Zhu, DL},
  journal={Optimization},
  volume={33},
  number={1},
  pages={9--27},
  year={1995},
  publisher={Taylor \& Francis}
}

@article{arbel2021amortized,
  title={Amortized implicit differentiation for stochastic bilevel optimization},
  author={Arbel, Michael and Mairal, Julien},
  journal={arXiv preprint arXiv:2111.14580},
  year={2021}
}

@inproceedings{shaban2019truncated,
  title={Truncated back-propagation for bilevel optimization},
  author={Shaban, Amirreza and Cheng, Ching-An and Hatch, Nathan and Boots, Byron},
  booktitle={The 22nd international conference on artificial intelligence and statistics},
  pages={1723--1732},
  year={2019},
  organization={PMLR}
}

@inproceedings{shen2024method,
  title={A method for bilevel optimization with convex lower-level problem},
  author={Shen, Han and Paternain, Santiago and Liu, Gaowen and Kompella, Ramana and Chen, Tianyi},
  booktitle={ICASSP 2024-2024 IEEE International Conference on Acoustics, Speech and Signal Processing (ICASSP)},
  pages={9426--9430},
  year={2024},
  organization={IEEE}
}

@article{chen2023bilevel,
  title={Bilevel optimization without lower-level strong convexity from the hyper-objective perspective},
  author={Chen, Lesi and Xu, Jing and Zhang, Jingzhao},
  year={2023}
}

@article{murty1987some,
  title={Some NP-complete problems in quadratic and nonlinear programming},
  author={Murty, Katta G and Kabadi, Santosh N and others},
  journal={Mathematical programming},
  volume={39},
  number={2},
  pages={117--129},
  year={1987}
}

@inproceedings{agarwal2017finding,
  title={Finding approximate local minima faster than gradient descent},
  author={Agarwal, Naman and Allen-Zhu, Zeyuan and Bullins, Brian and Hazan, Elad and Ma, Tengyu},
  booktitle={Proceedings of the 49th annual ACM SIGACT symposium on theory of computing},
  pages={1195--1199},
  year={2017}
}

@article{allen2018neon2,
  title={Neon2: Finding local minima via first-order oracles},
  author={Allen-Zhu, Zeyuan and Li, Yuanzhi},
  journal={Advances in Neural Information Processing Systems},
  volume={31},
  year={2018}
}

@article{jin2021nonconvex,
  title={On nonconvex optimization for machine learning: Gradients, stochasticity, and saddle points},
  author={Jin, Chi and Netrapalli, Praneeth and Ge, Rong and Kakade, Sham M and Jordan, Michael I},
  journal={Journal of the ACM (JACM)},
  volume={68},
  number={2},
  pages={1--29},
  year={2021},
  publisher={ACM New York, NY, USA}
}

@article{liu2018adaptive,
  title={Adaptive negative curvature descent with applications in non-convex optimization},
  author={Liu, Mingrui and Li, Zhe and Wang, Xiaoyu and Yi, Jinfeng and Yang, Tianbao},
  journal={Advances in Neural Information Processing Systems},
  volume={31},
  year={2018}
}

@article{fang2018spider,
  title={Spider: Near-optimal non-convex optimization via stochastic path-integrated differential estimator},
  author={Fang, Cong and Li, Chris Junchi and Lin, Zhouchen and Zhang, Tong},
  journal={Advances in neural information processing systems},
  volume={31},
  year={2018}
}

@article{vlatakis2019efficiently,
  title={Efficiently avoiding saddle points with zero order methods: No gradients required},
  author={Vlatakis-Gkaragkounis, Emmanouil-Vasileios and Flokas, Lampros and Piliouras, Georgios},
  journal={Advances in neural information processing systems},
  volume={32},
  year={2019}
}

@article{zhang2022zeroth,
  title={Zeroth-order negative curvature finding: Escaping saddle points without gradients},
  author={Zhang, Hualin and Xiong, Huan and Gu, Bin},
  journal={Advances in Neural Information Processing Systems},
  volume={35},
  pages={38332--38344},
  year={2022}
}

@inproceedings{ren2023escaping,
  title={Escaping saddle points in zeroth-order optimization: the power of two-point estimators},
  author={Ren, Zhaolin and Tang, Yujie and Li, Na},
  booktitle={International Conference on Machine Learning},
  pages={28914--28975},
  year={2023},
  organization={PMLR}
}

@inproceedings{zhang2022faster,
  title={Faster gradient-free methods for escaping saddle points},
  author={Zhang, Hualin and Gu, Bin},
  booktitle={The Eleventh International Conference on Learning Representations},
  year={2022}
}

@article{chen2025near,
  title={Near-optimal nonconvex-strongly-convex bilevel optimization with fully first-order oracles},
  author={Chen, Lesi and Ma, Yaohua and Zhang, Jingzhao},
  journal={Journal of Machine Learning Research},
  volume={26},
  number={109},
  pages={1--56},
  year={2025}
}

@inproceedings{pan2025scalebio,
  title={Scalebio: Scalable bilevel optimization for llm data reweighting},
  author={Pan, Rui and Zhang, Dylan and Zhang, Hanning and Pan, Xingyuan and Xu, Minrui and Zhang, Jipeng and Pi, Renjie and Wang, Xiaoyu and Zhang, Tong},
  booktitle={Proceedings of the 63rd Annual Meeting of the Association for Computational Linguistics (Volume 1: Long Papers)},
  pages={31959--31982},
  year={2025}
}

@inproceedings{wang2024helpsteer,
  title={Helpsteer: Multi-attribute helpfulness dataset for steerlm},
  author={Wang, Zhilin and Dong, Yi and Zeng, Jiaqi and Adams, Virginia and Sreedhar, Makesh Narsimhan and Egert, Daniel and Delalleau, Olivier and Scowcroft, Jane and Kant, Neel and Swope, Aidan and others},
  booktitle={Proceedings of the 2024 Conference of the North American Chapter of the Association for Computational Linguistics: Human Language Technologies (Volume 1: Long Papers)},
  pages={3371--3384},
  year={2024}
}

@article{meta2024,
  title={meta-llama/llama-3.2-3b-instruct},
  author={Meta},
  year={2024},
  url={https://huggingface.co/meta-llama/Llama-3.2-3B-Instruct}
}

@inproceedings{qin2025duet,
  title={DUET: Decentralized bilevel optimization without lower-level strong convexity},
  author={Qin, Zhen and Liu, Zhuqing and Lu, Songtao and Liang, Yingbin and Liu, Jia},
  booktitle={The Thirteenth International Conference on Learning Representations},
  year={2025}
}

@article{malladi2023fine,
  title={Fine-tuning language models with just forward passes},
  author={Malladi, Sadhika and Gao, Tianyu and Nichani, Eshaan and Damian, Alex and Lee, Jason D and Chen, Danqi and Arora, Sanjeev},
  journal={Advances in Neural Information Processing Systems},
  volume={36},
  pages={53038--53075},
  year={2023}
}

@article{castin2023smooth,
  title={How smooth is attention?},
  author={Castin, Val{\'e}rie and Ablin, Pierre and Peyr{\'e}, Gabriel},
  journal={arXiv preprint arXiv:2312.14820},
  year={2023}
}

@article{li2024getting,
  title={Getting more juice out of the sft data: Reward learning from human demonstration improves sft for llm alignment},
  author={Li, Jiaxiang and Zeng, Siliang and Wai, Hoi-To and Li, Chenliang and Garcia, Alfredo and Hong, Mingyi},
  journal={Advances in Neural Information Processing Systems},
  volume={37},
  pages={124292--124318},
  year={2024}
}

@article{hu2023contextual,
  title={Contextual stochastic bilevel optimization},
  author={Hu, Yifan and Wang, Jie and Xie, Yao and Krause, Andreas and Kuhn, Daniel},
  journal={Advances in Neural Information Processing Systems},
  volume={36},
  pages={78412--78434},
  year={2023}
}
\bibliographystyle{apalike}

\clearpage
\appendix
\section*{Appendix}

\section{Additional Numerical Results}\label{sec:app-exp}

\subsection{Meta-Learning Experiments}

Meta-learning seeks to improve model generalization and is naturally formulated as a BLO problem \citep{franceschi2018bilevel,ji2021bilevel,hu2023contextual,qin2025duet}. Specifically, the lower level implements task-specific adaptation on the training data, while the upper level enhances the representation utility to ensure generalization ability of the model.

We train a multi-layer perceptron (MLP, corresponding to $x$) with depth $2$, width $512$ and ReLU activation, along with an additional $256$-dimensional linear layer (corresponding to $y$) on MNIST dataset. In particular, we construct 5 MNIST heterogeneous training datasets, each containing $80\%$ of a selected pair of digits (e.g., $(0,1)$) and $20\%$ of other digit pairs. The validation set are constructed from the MNIST test split, using the same ratio of five digit-pair tasks. The formal formulation is given as follows:
\begin{align*}
    & \min_{x,y} \sum_{s\in\mathcal{D}_0}\mathcal{L}\big((x\circ y)(s), l(s)\big) \\
    \text{subject to } & y \in \mathcal{S}(x) := \argmin_y \sum_{j=1}^5 \sum_{s\in\mathcal{D}_j}\mathcal{L}\big((x\circ y)(s), l(s)\big),
\end{align*}
where $s$ denotes the digit sample, $l(s)$ denotes its label, $x\circ y(s)$ denotes the model output, $\mathcal{L}$ denotes the cross-entropy loss function, and $\mathcal{D}_j, j\in\{0, \dotsc,5\}$ are the datasets. Since the loss function is convex, and $y$ corresponds to a linear layer, the lower-level problem is convex (but not strongly convex) w.r.t. $y$. We consider this meta-learning task to demonstrate the capability of \alg in the LLGC setting.

As shown in \Cref{fig:app-meta}, all baselines and our \alg converge at both upper and lower levels. All baseline methods succeed in this meta-learning task because the FOSP-based relaxed surrogate aligns exactly with the original BLO problem in this LLGC environment, preventing them from getting trapped at saddle points (In fact, no saddle points exist due to the convexity). \algn, on the other hand, still achieves the fastest convergence, demonstrating both its validity in the LLGC setting and its consistent efficiency.

\begin{figure}[H]
    \centering
    \begin{subfigure}[t]{0.425\linewidth}
        \centering
        \includegraphics[width=\textwidth]{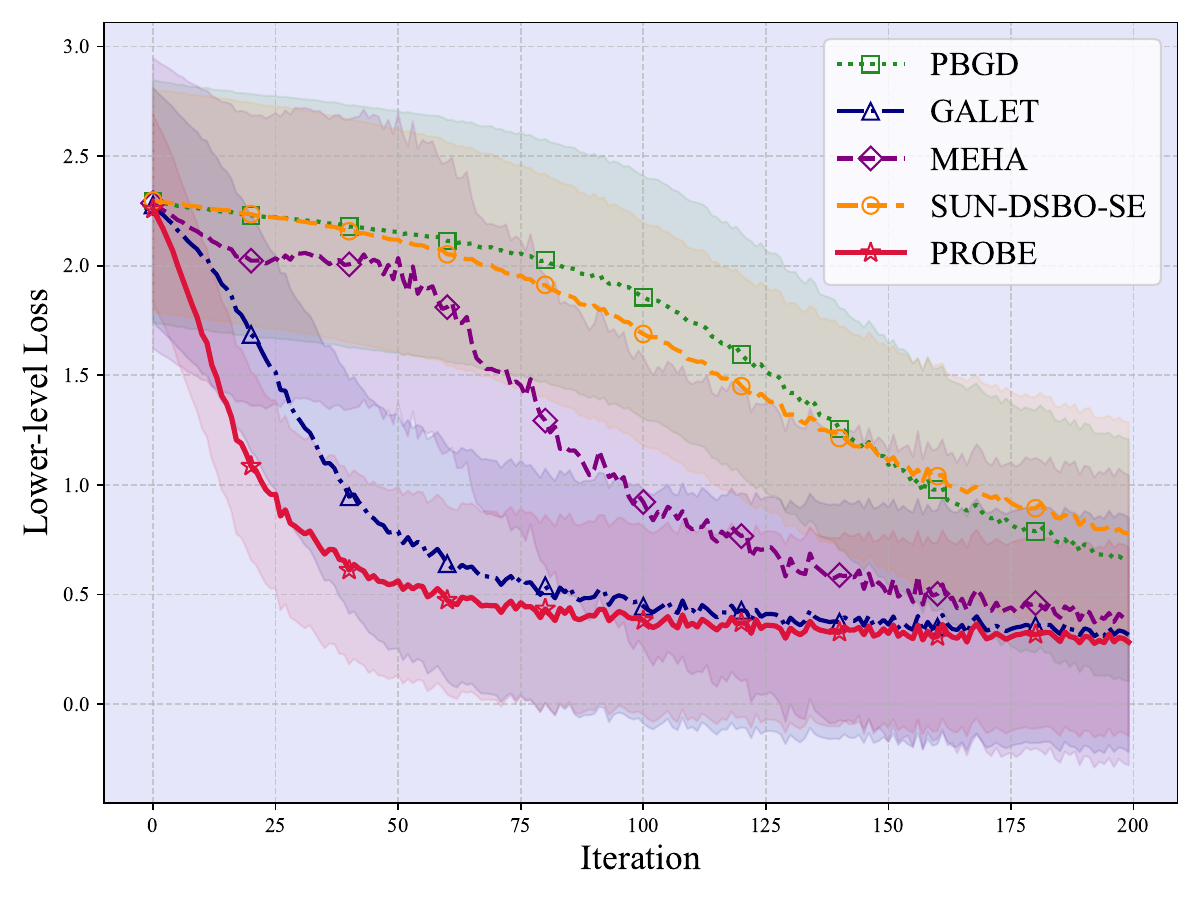}
        \vspace{-1.5em}
        \caption{Lower-level loss.}
        \label{fig:meta-ll-iteration}
    \end{subfigure}
    \hfill
    \begin{subfigure}[t]{0.425\linewidth}
        \centering
        \includegraphics[width=\textwidth]{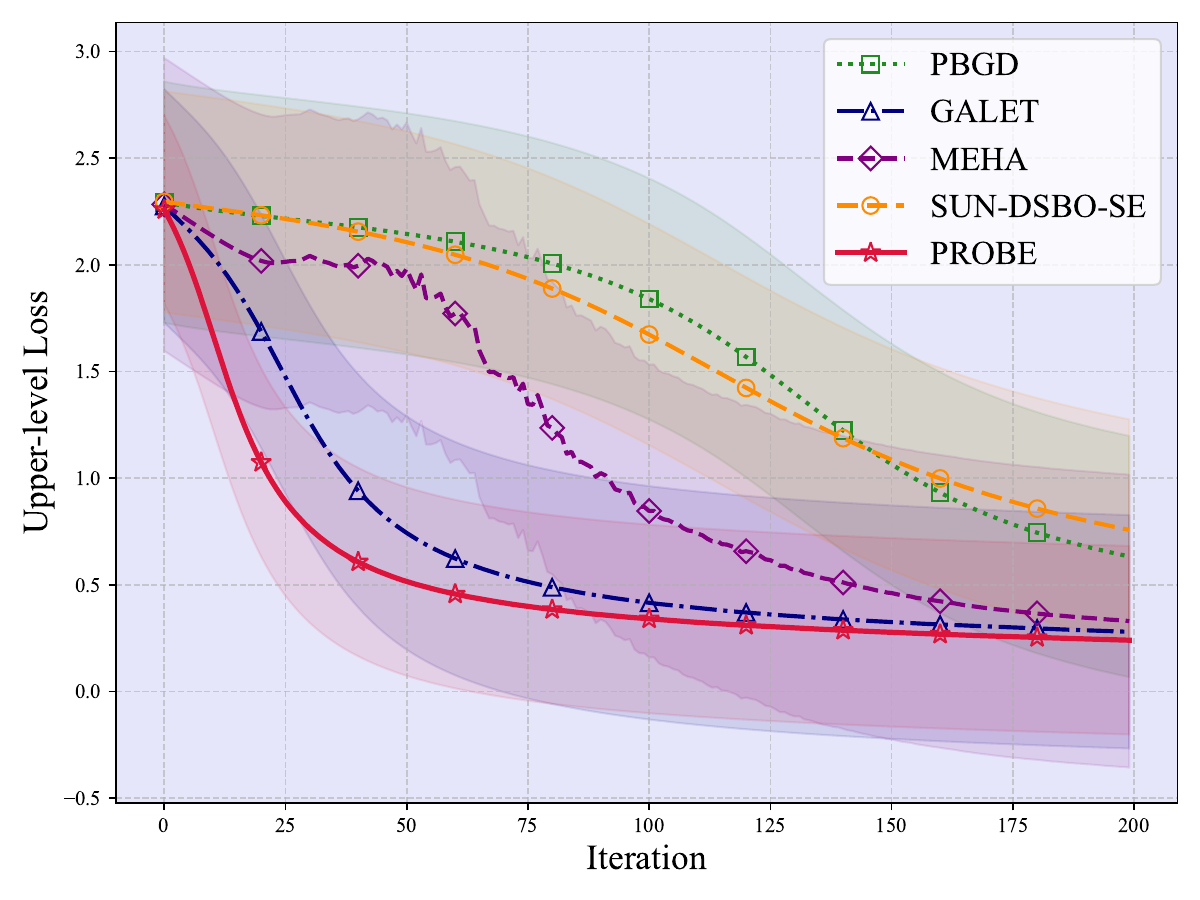}
        \vspace{-1.5em}
        \caption{Upper-level loss.}
        \label{fig:meta-ll-time}
    \end{subfigure}
    \caption{Baseline comparison in meta-learning task.}
    \label{fig:app-meta}
\end{figure}

\subsection{Implementation Details and Additional Results on Data Curation Task}

In the data curation task, we fine-tune an LLM, Llama-3.2-3B-Instruct \citep{meta2024} model, on the HelpSteer \citep{wang2024helpsteer} dataset, which naturally contains prompts with responses of varying quality. During training, corresponding to the lower-level problem in our formulation, the LLM learns to distinguish between two datasets and is expected to assign higher weights to the one with higher response quality. During validation, corresponding to the upper-level problem, the performance of the fine-tuned LLM is evaluated on high-quality prompt-response data. Specifically, the datasets are constructed as follows: for each response data in HelpSteer, we compute its average score, and include the corresponding prompt-response pair in the high-quality set if $s \ge 2.5$, or in the low-quality set if $s \le 2$ (the only validation set merely include high-quality validation data). Llama-3.2-3B-Instruct is trained using LoRA technique with $\text{rank}=8$, and all data curation experiments are conducted on a cluster of 2 NVIDIA H200 GPUs (approximately 140GB each) using PyTorch's DistributedDataParallel.

The hyperparameter selection is detailed as follows. The batch size is set to $32$. For a fair comparison, we fix the total number of iterations to $T = 20,000$, which implies: (i) for single-loop methods, the algorithm runs for $20,000$ rounds; and (ii) for double-loop methods, the product of inner-loop steps and outer-loop steps equals $20,000$. Additionally, we simplify our method by using fixed values for $K$, $N$, and $\mu$ to ensure fairness in the comparison. Accordingly, we uniformly select $100$ iterations to report. The learning rates for all algorithms are set to $10^{-5}$. For \algn, we set $T=100$, $K=200$, $\mu=0.1$. We set $\Delta_K = 10$, and other PGD parameters are set to $10^{-2}$ as our default setup. For these baselines, the detailed setups are specified as follows. For MEHA, we examined additional learning rate $\eta \in \{0.1, 2.0\}$ and penalty coefficient $c \in \{0.1, 1, 2, 5, 10\}$. For SUN-DSBO-SE, we examined proximal parameter $\gamma \in \{0.01, 0.05, 0.1\}$ and penalty parameter $\mu_0 \in \{0.01, 0.1, 1, 10\}$. For PBGD, we examined the penalty coefficient in $\{0.1, 1, 10\}$. For GALET, we also tested larger inner-loop step counts $(10, 20)$, but these configurations resulted in numerical errors; we therefore used the best stable configuration with both inner-loop counts set to $5$. The reported baseline performances correspond to the best-performing configurations among those tested rather than unfavorable or arbitrarily selected settings. Each experiment is repeated $5$ times, and the corresponding standard error bars are shown in all figures. We also apply an exponential moving average (EMA) to enhance visibility.

The lower-level performances of all methods are shown in \Cref{fig:baseline-app}. As noted earlier, GALET finds some seemingly favorable solutions for the lower-level problem (Its loss values also decrease from around $8$, and the seemingly immediate convergence observed in the figure is due to the use of EMA). Nevertheless, when considered alongside \Cref{fig:baseline}, it is clear that these solutions are actually suboptimal in general. In stark contrast, MEHA, SUN-DSBO-SE, and our \alg not only converge during training, but also performs consistently well in the validation process. Among them, our \alg is the most efficient method, achieving convergence in the shortest running time.

\begin{figure}[t]
    \centering
    \begin{subfigure}[t]{0.425\linewidth}
        \centering
        \includegraphics[width=\textwidth]{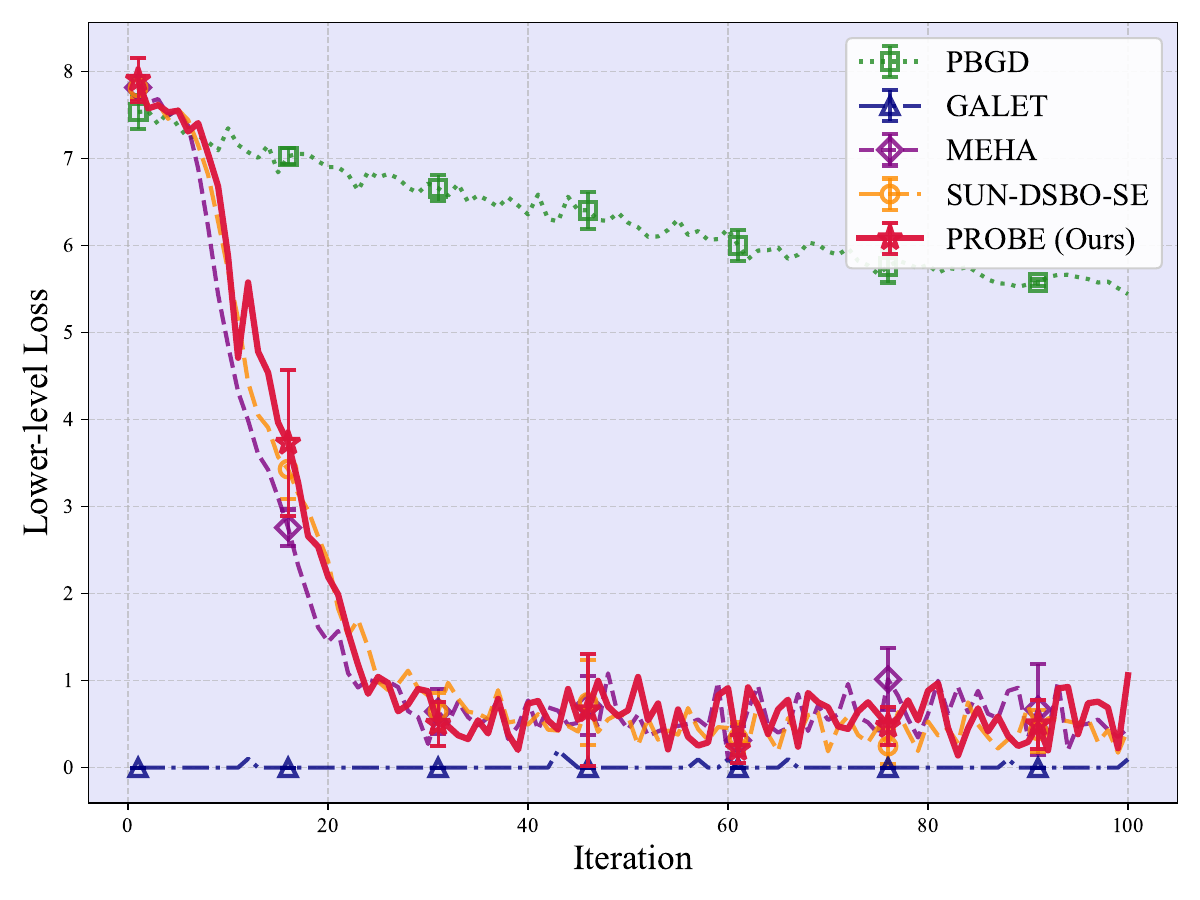}
        \vspace{-1.5em}
        \caption{Loss vs Iteration.}
        \label{fig:baseline-ll-iteration}
    \end{subfigure}
    \hfill
    \begin{subfigure}[t]{0.425\linewidth}
        \centering
        \includegraphics[width=\textwidth]{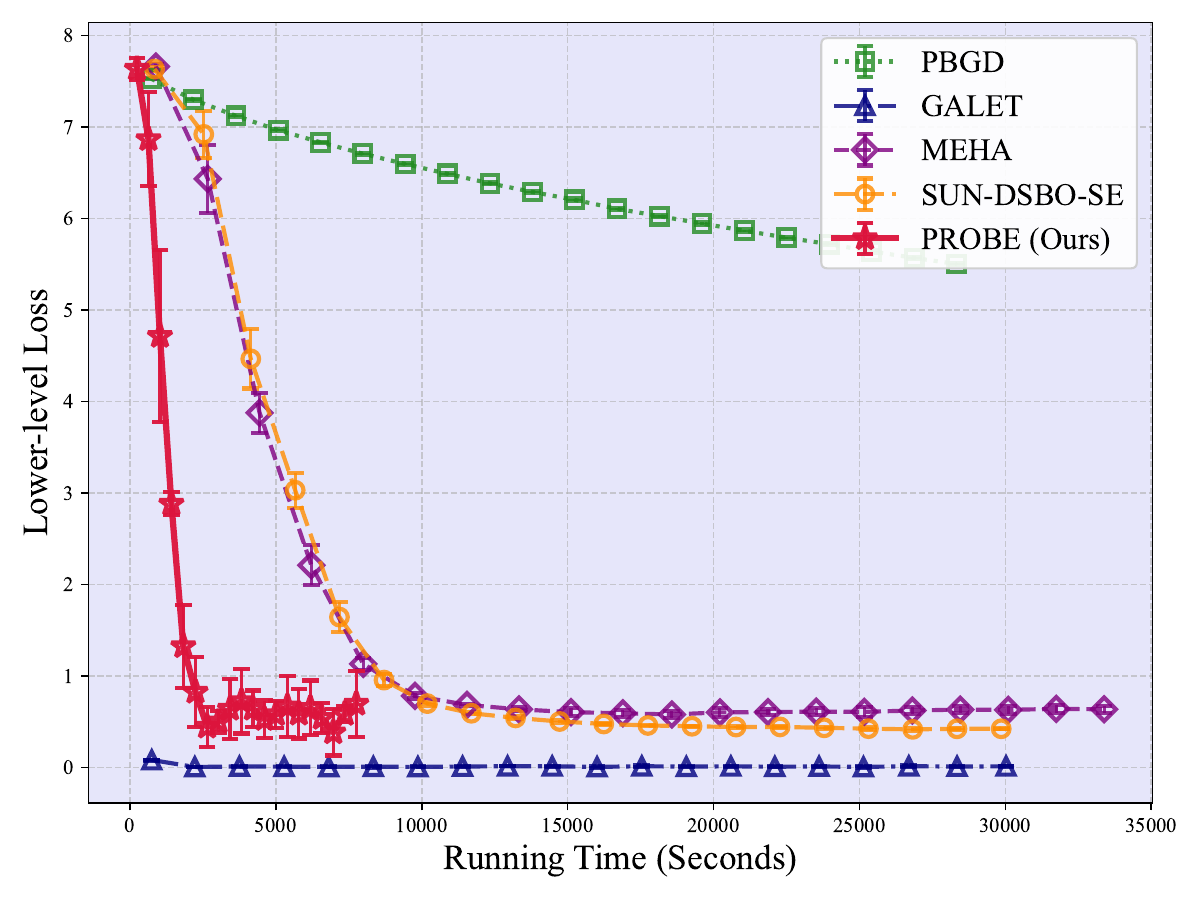}
        \vspace{-1.5em}
        \caption{Loss vs Running time.}
        \label{fig:baseline-ll-time}
    \end{subfigure}
    \caption{Baseline (lower-level) comparison.}
    \label{fig:baseline-app}
\end{figure}

\section{Proof of Main Result}\label{sec:app-proof}

\begin{theorem}[Formally restated \Cref{thm:main}]
    By selecting parameters as follows:
    \begingroup
    \allowdisplaybreaks[4]
    \begin{align*}
        & \sigma_t=\Theta\left((t+1)^{-p}\right), \,\,\,\,\,\, \epsilon_t = \frac{\sigma_t^4}{\rho}, \,\,\,\,\,\, \mu_t = \sqrt{\rho\epsilon_t} + \sigma_t, \,\,\,\,\,\, \alpha_t=\Theta\left((t+1)^{-q}\right), \,\,\,\,\,\, \beta_t = \Theta\left(\frac{1}{\ell}\right), \\
        & N_t = \widetilde{\Theta}((t+1)^{\frac{p}{2}}), \,\,\,\,\,\, K_t = \widetilde{\Theta}((t+1)^{8p}), \\
        & \delta'\in(0,\frac{1}{T}), \,\,\,\,\,\, \Delta_{g,t} \ge g(x_t,y_t)-\inf_y g(x_t,y), \,\,\,\,\,\, \chi_t=3\max\left\{\log\left(\frac{q\ell\Delta_{g,t}}{\epsilon_t^2\delta'}\right), 4\right\}, \\
        & r_t = \Theta\left(\frac{\epsilon_t}{\ell\chi_t^2}\right), \,\,\,\,\,\, \varrho_t = \Theta\left(\frac{\epsilon_t}{\chi_t^2}\right), \,\,\,\,\,\, \Delta_{G,t} = \Theta\left(\frac{\epsilon_t^{1.5}}{\chi_t^3\rho^{0.5}}\right), \,\,\,\,\,\, \Delta_{K,t} = \Theta\left(\frac{\ell\chi_t}{\rho^{0.5}\epsilon_t^{0.5}}\right),
    \end{align*}
    \endgroup
    for any $\delta = T\delta'\in(0,1)$, $p\in(0,\frac{1}{3})$, $3p \le q < 1$, the sequence $\{x_t,y_t\}_t$ generated by \Cref{alg:main} satisfies the following convergence guarantee w.p. $1-\delta$:
    \begin{align*}
        \min_{\frac{T}{2}\le t\le T} \Pi(x_t, y_t) = \mathcal{O}\left( \frac{1}{T^{1-q}} + \frac{1}{T^{2p}} + \frac{1}{T^{\frac{4}{5}}} \right).
    \end{align*}
    When selecting $p=\frac{1}{5}$ and $q=\frac{3}{5}$, we have $\min_{\frac{T}{2}\le t\le T} \Pi(x_t, y_t) = \mathcal{O}\left( T^{-\frac{2}{5}} \right)$.
\end{theorem}

\begin{proof}
    We prove this result with several sub-steps.

    \textbf{Part $\mathcal{A}$.} We first control $\| v_{\mu_{t+1}}^*(x_t) - v_{\mu_t}^*(x_t) \|$. To begin with, we define:
    \begin{equation*}
        v_\mu^*(x_t) := - \left[ \nabla_{yy}^2 g(x_t, y_{\mu}^*(x_t)) + \mu I_q \right]^{-1} \nabla_y f(x_t, y_{\mu}^*(x_t)),
    \end{equation*}
    For any fixed $t\ge0$, let $H_\mu = \nabla_{yy}^2 g(x_t, y_\mu^*(x_t)) + \mu I_q$ and $w_\mu = \nabla_y f(x_t, y_\mu^*(x_t))$. Then, we have:
    \begin{equation}\label{eq:pf_1}
        \begin{aligned}
            \|v_{\mu_{t+1}}^*(x_t) - v_{\mu_t}^*(x_t)\| & = \|H_{\mu_{t+1}}^{-1} w_{\mu_{t+1}} - H_{\mu_t}^{-1} w_{\mu_t}\| \\
            & \le \|H_{\mu_{t+1}}^{-1} (w_{\mu_{t+1}} - w_{\mu_t})\| + \|(H_{\mu_{t+1}}^{-1} - H_{\mu_t}^{-1}) w_{\mu_t}\| \\
            & \overset{\flat}{\le} \frac{\ell}{\sigma_{t+1}} \|y_{\mu_{t+1}}^*(x_t) - y_{\mu_t}^*(x_t)\| + \|H_{\mu_{t+1}}^{-1}\| \cdot \|H_{\mu_t} - H_{\mu_{t+1}}\| \cdot \|H_{\mu_t}^{-1}\| \cdot \nu,
        \end{aligned}
    \end{equation}
    where $\flat$ is due to \Cref{ass:smoothness} and $A^{-1}-B^{-1} = A^{-1}(B-A)B^{-1}$. For the second term, we have:
    \begin{equation*}
        \|H_{\mu_t} - H_{\mu_{t+1}}\| \le \rho \|y_{\mu_t}^*(x_t) - y_{\mu_{t+1}}^*(x_t)\| + |\mu_t - \mu_{t+1}|.
    \end{equation*}

    Thus, to tackle \Cref{eq:pf_1}, we need to control the gap $\|y_{\mu_t}^*(x_t) - y_{\mu_{t+1}}^*(x_t)\|$. To this end, we first define the following supporting function for any fixed $t$:
    \begin{equation*}
        h_{\mu}(y) = g(x_t, y) + \frac{\mu}{2} \| y - y_{t+1} \|^2.
    \end{equation*}
    Since our Alg.~\ref{alg:PGD} ensures that w.p. $1-\delta'$, $y_{t+1}$ is an $\epsilon_t$-SOSP, we know that $\|\nabla_y g(x_t, y_{t+1})\| \le \epsilon_t$ and $\lambda_{\min}(\nabla_{yy}^2 g(x_t, y_{t+1})) \ge -\sqrt{\rho\epsilon_t}$. According to our parameter selection, we can get $\mu_t = \Theta(\sigma_t)$ and $\sqrt{\rho\epsilon_t} = \sigma_t^2$. Accordingly, we have:
    \begin{align*}
        &\lambda_{\min}(\nabla^2 h_{\mu_t}(y_{t+1})) \ge \mu_t - \sqrt{\rho\epsilon_t} = := \iota_t, \\
        &\lambda_{\min}(\nabla^2 h_{\mu_{t+1}}(y_{t+1})) \ge \mu_{t+1} - \sqrt{\rho\epsilon_t} := \tilde{\iota}_{t+1}.
    \end{align*}
    Note that $\sigma_t = \Theta((t+1)^{-p})$ with $p\in(0,\frac{1}{3})$, we can easily get $\iota_t = \Theta(\sigma_t)$ and $\tilde{\iota}_{t+1} = \Theta(\sigma_t)$. Therefore, within a sufficiently small ball $\mathcal{B}_{y_{t+1}}(\frac{\iota_t}{2\rho})$, we have $\lambda_{\min}(\nabla^2 h_{\mu_t}(y)) \ge \frac{\iota_t}{2}$. Thus, noting that $\nabla h_{\mu_t}(y_{\mu_t}^*(x_t))=0$, we can obtain:
    \begin{align*}
        & \|y_{\mu_t}^*(x_t) - y_{t+1}\| \le \frac{2}{\iota_t}\|\nabla h_{\mu_t}(y_{t+1})\| = \frac{2}{\iota_t}\|\nabla_y g(x_t, y_{t+1})\| \le \frac{2\epsilon_t}{\iota_t} = \frac{2\sigma_t^4 / \rho}{\iota_t}, \\
        \implies & \|y_{\mu_t}^*(x_t) - y_{t+1}\| = \mathcal{O}(\sigma_t^3).
    \end{align*}

    Similarly, within a small ball centered at $y_{t+1}$, $h_{\mu_{t+1}}$ remains $\frac{\tilde{\iota}_{t+1}}{2}$-strong convex, which implies:
    \begin{align*}
        \frac{\tilde{\iota}_{t+1}}{2} \|y_{\mu_{t+1}}^*(x_t) - y_{\mu_t}^*(x_t)\|^2 & \le \langle \nabla h_{\mu_{t+1}}(y_{\mu_t}^*(x_t)) - \nabla h_{\mu_{t+1}}(y_{\mu_{t+1}}^*(x_t)), y_{\mu_t}^*(x_t) - y_{\mu_{t+1}}^*(x_t) \rangle \\
        & \overset{\dagger}{=} \langle \nabla_y g(x_t, y_{\mu_t}^*(x_t)) + \mu_{t+1}(y_{\mu_t}^*(x_t) - y_{t+1}), y_{\mu_t}^*(x_t) - y_{\mu_{t+1}}^*(x_t) \rangle \\
        & \overset{\ddagger}{=} \langle (\mu_{t+1} - \mu_t)(y_{\mu_t}^*(x_t) - y_{t+1}), y_{\mu_t}^*(x_t) - y_{\mu_{t+1}}^*(x_t) \rangle \\
        & \overset{\flat}{\le} |\mu_{t+1} - \mu_t| \cdot \|y_{\mu_t}^*(x_t) - y_{t+1}\| \cdot \|y_{\mu_{t+1}}^*(x_t) - y_{\mu_t}^*(x_t)\|,
    \end{align*}
    where $\dagger$ is because $y_{\mu_{t+1}}^*(x_t)$ minimizes $h_{\mu_{t+1}}$, $\ddagger$ is due to $\nabla_y g(x_t, y_{\mu_t}^*(x_t)) + \mu_t(y_{\mu_t}^*(x_t) - y_{t+1}) = 0$, and $\flat$ is induced by the Cauchy-Schwarz inequality. Thus, we have:
    \begin{equation*}
        \|y_{\mu_{t+1}}^*(x_t) - y_{\mu_t}^*(x_t)\| \le \frac{2|\mu_{t+1} - \mu_t|}{\tilde{\iota}_{t+1}} \|y_{\mu_t}^*(x_t) - y_{t+1}\|.
    \end{equation*}
    Substituting the bound $\|y_{\mu_t}^*(x_t) - y_{t+1}\| = \mathcal{O}(\sigma_t^3)$ and applying the rigorous lower bound $\tilde{\iota}_{t+1}=\Theta(\sigma_t)$, we can get:
    \begin{equation*}
        \|y_{\mu_{t+1}}^*(x_t) - y_{\mu_t}^*(x_t)\| = \mathcal{O}(\sigma_t^2) |\mu_t - \mu_{t+1}|.
    \end{equation*}

    We then substitute the above results into \Cref{eq:pf_1} to get:
    \begin{align*}
        & \|v_{\mu_{t+1}}^*(x_t) - v_{\mu_t}^*(x_t)\| \le \frac{\ell}{\sigma_{t+1}} \|y_{\mu_{t+1}}^*(x_t)\! -\! y_{\mu_t}^*(x_t)\| + \|H_{\mu_{t+1}}^{-1}\| \cdot \|H_t - H_{t+1}\| \cdot \|H_{\mu_t}^{-1}\| \cdot \nu \\
        \implies & \|v_{\mu_{t+1}}^*(x_t) - v_{\mu_t}^*(x_t)\| = \left( \frac{\ell}{\sigma_{t+1}} + \frac{\nu\rho}{\sigma_t\sigma_{t+1}} \right) \mathcal{O}(\sigma_t^2) |\mu_t - \mu_{t+1}| + \frac{\nu}{\sigma_t\sigma_{t+1}}|\mu_t - \mu_{t+1}|.
    \end{align*}
    This indicates:
    \begin{align*}
        \|v_{\mu_{t+1}}^*(x_t) - v_{\mu_t}^*(x_t)\| & = \mathcal{O}\left(\frac{1}{\sigma_t\sigma_{t+1}}|\mu_t - \mu_{t+1}|\right) = \mathcal{O}\left( \frac{(t+1)^{-p-1}}{(t+1)^{-p}(t+2)^{-p}} \right) \\
        & = \mathcal{O}\left( \frac{(t+1)^{-p-1}}{(t+1)^{-2p}2^{-p}} \right) = \mathcal{O}\big((t+1)^{p-1}\big).
    \end{align*}

    \textbf{Part $\mathcal{B}$.} We then consider the error introduced by the Conjugate Gradient step. Recall that $v_{\mu_t}^*(x_t) = [\nabla_{yy}^2 g(x_t, y_{\mu_t}^*(x_t)) + \mu_t I_q]^{-1} \nabla_y f(x_t, y_{\mu_t}^*(x_t))$, and we now define $E_t := \|v_t - v_{\mu_{t-1}}^*(x_{t-1})\|^2$. Note that Conjugate Gradient actually solves a linear system $H_t v = b_t$, where $H_t = \nabla_{yy}^2 g(x_t, y_{t+1}) + \mu_t I_q$ and $b_t = \nabla_y f(x_t, y_{t+1})$ at each step $t$. We correspondingly denote the exact solution to this linear system as $v_t^\circ = H_t^{-1} b_t$. According to \cite{grazzi2020iteration}, by selecting $N_t = \widetilde{\Theta}(\sqrt{(t+1)^p})$, we can obtain:
    \begin{align*}
        & \|v_{t+1} - v_t^\circ\| \le 2(t+1)^{\frac{p}{2}} \left( \frac{(t+1)^{\frac{p}{2}} - 1}{(t+1)^{\frac{p}{2}} + 1} \right)^{N_t} \|v_t - v_t^\circ\|, \,\,\, \implies \,\,\, \|v_{t+1} - v_t^\circ\| \le \frac{1}{\sqrt{2}} \|v_t - v_t^\circ\|.
    \end{align*}
    Therefore, we have:
    \begin{align*}
        & \sqrt{E_{t+1}} = \|v_{t+1} - v_{\mu_t}^*(x_t)\| \le \|v_{t+1} - v_t^\circ\| + \|v_t^\circ - v_{\mu_t}^*(x_t)\|.
    \end{align*}
    Thus, we need to bound $\|v_t - v_t^\circ\|$, which can be controlled by:
    \begin{align*}
        \|v_t - v_t^\circ\| \le \sqrt{E_t} + \underbrace{\|v_{\mu_{t-1}}^*(x_{t-1}) - v_{\mu_{t-1}}^*(x_t)\|}_{\Circled{1}} + \underbrace{\|v_{\mu_{t-1}}^*(x_t) - v_{\mu_t}^*(x_t)\|}_{\Circled{2}} + \underbrace{\|v_{\mu_t}^*(x_t) - v_t^\circ\|}_{\Circled{3}}.
    \end{align*}
    We now control each of these terms. For \Circled{1}, we first note that:
    \begin{align*}
        & \|v_{\mu}^*(x) - v_{\mu}^*(x')\| \\
        \le &  \left\| \left[ \nabla_{yy}^2 g(x, y_{\mu}^*(x)) + \mu I_q \right]^{-1} \big( \nabla_y f(x, y_{\mu}^*(x)) - \nabla_y f(x', y_{\mu}^*(x')) \big) \right\| \\
        & + \left\| \left( \left[ \nabla_{yy}^2 g(x, y_{\mu}^*(x)) + \mu I_q \right]^{-1} - \left[ \nabla_{yy}^2 g(x', y_{\mu}^*(x')) + \mu I_q \right]^{-1} \right) \nabla_y f(x', y_{\mu}^*(x')) \right\| \\
        \le & \left( \frac{\ell}{\sigma} + \frac{\nu\rho}{\sigma^2} \right) \left( 1 + \frac{\ell}{\sigma} \right) \|x-x'\| = \Theta (\sigma^{-3}) \|x-x'\|,
    \end{align*}
    meaning that its Lipschitz continuity coefficient is in the order of $\Theta (\sigma^{-3})$. Therefore, we have:
    \begin{align*}
        \Circled{1} & \le \Theta (\sigma_t^{-3}) \alpha_{t-1} \| \widehat{\nabla}\Phi_{\mu_{t-1}}(x_{t-1}) \| \,\,\, \implies \,\,\, \Circled{1} = \mathcal{O}\big((t+1)^{3p-q}\big)\| \widehat{\nabla}\Phi_{\mu_{t-1}}(x_{t-1}) \|.
    \end{align*}
    For \Circled{2}, we can follow Part $\mathcal{A}$ to get $\Circled{2} = \mathcal{O}((t+1)^{p-1})$.

    For \Circled{3}, we have:
    \begin{align*}
        \Circled{3} & \le \left\| \left[ \nabla_{yy}^2 g(x_t, y_{\mu_t}^*(x_t)) + \mu_t I_q \right]^{-1} \big( \nabla_y f(x_t, y_{\mu_t}^*(x_t)) - \nabla_y f(x_t, y_{t+1}) \big) \right\| \\
        & \,\,\,\,\,\, + \left\| \left( \left[ \nabla_{yy}^2 g(x_t, y_{\mu_t}^*(x_t)) + \mu_t I_q \right]^{-1} - H_t^{-1} \right) \nabla_y f(x_t, y_{t+1}) \right\|, \\
        \implies \Circled{3} & = \frac{\ell}{\sigma_t}\mathcal{O}(\sigma_t^3) + \frac{\nu\rho}{\sigma_t^2}\mathcal{O}(\sigma_t^3) = \mathcal{O}\big( (t+1)^{-p} \big).
    \end{align*}

    Substituting them back, we can get:
    \begin{align*}
        & \|v_t - v_t^\circ\| = \sqrt{E_t} + \mathcal{O}\big((t+1)^{3p-q}\big)\| \widehat{\nabla}\Phi_{\mu_{t-1}}(x_{t-1}) \| + \mathcal{O}\big( (t+1)^{-p} \big), \\
        \implies & \sqrt{E_{t+1}} = \frac{1}{\sqrt{2}} \sqrt{E_t} + \mathcal{O}\big((t+1)^{3p-q}\big)\| \widehat{\nabla}\Phi_{\mu_{t-1}}(x_{t-1}) \| + \mathcal{O}\big( (t+1)^{-p} \big).
    \end{align*}
    According to the Young's inequality with constant $1/2$, we have:
    \begin{align*}
        E_{t+1} & = \left(1+\frac{1}{2}\right) \left( \frac{1}{\sqrt{2}} \sqrt{E_t} \right)^2 + (1+2) \left( \mathcal{O}\big((t+1)^{3p-q}\big)\| \widehat{\nabla}\Phi_{\mu_{t-1}}(x_{t-1}) \| + \mathcal{O}\big( (t+1)^{-p} \big) \right)^2 \\
        & = \frac{3}{4} E_t + \mathcal{O}\big((t+1)^{6p-2q}\big)\| \widehat{\nabla}\Phi_{\mu_{t-1}}(x_{t-1}) \|^2 + \mathcal{O}\big( (t+1)^{-2p} \big).
    \end{align*}

    \textbf{Part $\mathcal{C}$.} In this part, we control the shift $\Phi_{\mu_{t+1}}(x_{t+1}) - \Phi_{\mu_t}(x_t)$. We first compute the smoothness coefficient of $\Phi_\mu$ as follows. According to \Cref{ass:smoothness}, for any $x,x'$ we have:
    \begin{align*}
        & \| \nabla \Phi_\mu(x) - \nabla \Phi_\mu(x') \| \\
        \le & \| \nabla_x f(x, y_{\mu}^*(x)) - \nabla_x f(x', y_{\mu}^*(x')) \| \\
        & + \left\| \big( \nabla_{xy}^2 g(x, y_{\mu}^*(x)) - \nabla_{xy}^2 g(x', y_{\mu}^*(x')) \big) v_\mu^*(x) \right\| + \left\| \nabla_{xy}^2 g(x', y_{\mu}^*(x')) \big(v_\mu^*(x) - v_\mu^*(x')\big) \right\| \\
        \le & \big(\ell + \|v_\mu^*(x)\| \cdot \rho\big) \big( \|x-x'\| + \|y_{\mu}^*(x)-y_{\mu}^*(x')\| \big) + \ell \| v_\mu^*(x) - v_\mu^*(x') \| \\
        \le & \left(\ell + \frac{\nu}{\sigma} \cdot \rho\right) \big( \|x-x'\| + \|y_{\mu}^*(x)-y_{\mu}^*(x')\| \big) \\
        & + \ell \left\| \left[ \nabla_{yy}^2 g(x, y_{\mu}^*(x)) + \mu I_q \right]^{-1} \big( \nabla_y f(x, y_{\mu}^*(x)) - \nabla_y f(x', y_{\mu}^*(x')) \big) \right\| \\
        & + \ell \left\| \left( \left[ \nabla_{yy}^2 g(x, y_{\mu}^*(x)) + \mu I_q \right]^{-1} - \left[ \nabla_{yy}^2 g(x', y_{\mu}^*(x')) + \mu I_q \right]^{-1} \right) \nabla_y f(x', y_{\mu}^*(x')) \right\| \\
        \le & \left(\ell + \frac{\nu \rho}{\sigma} + \frac{\ell^2}{\sigma} + \frac{\nu \ell \rho}{\sigma^2}\right) \left( \|x-x'\| + \frac{\ell}{\sigma} \|x-x'\| \right) \\
        \le & \left(\ell + \frac{\nu \rho}{\sigma} + \frac{\ell^2}{\sigma} + \frac{\nu \ell \rho}{\sigma^2}\right)\left(1 + \frac{\ell}{\sigma}\right) \|x-x'\| \triangleq L_\mu \|x-x'\|.
    \end{align*}
    We denote $L_t:=L_{\mu_t}$. Given the $L_{t+1}$-smoothness of $\Phi_{\mu_{t+1}}$ with $L_{t+1} = \Theta((t+1)^{3p})$, and applying the update rule $x_{t+1} = x_t - \alpha_t \widehat{\nabla}\Phi_{\mu_t}(x_t)$ with $q\ge3p$, the descent lemma yields:
    \begin{align*}
        & \Phi_{\mu_{t+1}}(x_{t+1}) - \Phi_{\mu_{t+1}}(x_t) \\
        \le & \langle \nabla\Phi_{\mu_{t+1}}(x_t), x_{t+1} - x_t \rangle + \frac{L_{t+1}}{2} \|x_{t+1} - x_t\|^2 \\
        = & -\alpha_t \langle \nabla\Phi_{\mu_{t+1}}(x_t), \widehat{\nabla}\Phi_{\mu_t}(x_t) \rangle + \frac{L_{t+1} \alpha_t^2}{2} \|\widehat{\nabla}\Phi_{\mu_t}(x_t)\|^2 \\
        \le & -\frac{\alpha_t}{2}\|\nabla\Phi_{\mu_t}(x_t)\|^2 -\frac{\alpha_t}{4}\|\widehat{\nabla}\Phi_{\mu_t}(x_t)\|^2 + \frac{\alpha_t}{2}\|\nabla\Phi_{\mu_{t+1}}(x_t) - \widehat{\nabla}\Phi_{\mu_t}(x_t)\|^2,
    \end{align*}
    where in the last inequality, we use the fact that $\|a-b\|^2 = \|a\|^2 + \|b\|^2 - 2\langle a, b\rangle$. For the last term, we can control it by:
    \begin{align*}
        \|\nabla\Phi_{\mu_{t+1}}(x_t) - \widehat{\nabla}\Phi_{\mu_t}(x_t)\| \le & \underbrace{\|\nabla_x f(x_t, y_{\mu_{t+1}}^*(x_t)) - \nabla_x f(x_t, y_{t+1})\|}_{\text{Term A}} \\
        & + \underbrace{\|\nabla_{xy}^2 g(x_t, y_{\mu_{t+1}}^*(x_t)) v_{\mu_{t+1}}^*(x_t) - \nabla_{xy}^2 g(x_t, y_{t+1}) v_{t+1}\|}_{\text{Term B}}.
    \end{align*}
    Then, we have:
    \begin{align*}
        \text{Term A} & \le \ell \|y_{\mu_{t+1}}^*(x_t) - y_{t+1}\| = \mathcal{O} \big( \|y_{\mu_{t+1}}^*(x_t) - y_{\mu_t}^*(x_t)\| + \|y_{\mu_t}^*(x_t) - y_{t+1}\| \big), \\
        \implies \text{Term A} & = \mathcal{O}(t^{-3p-1}) + \mathcal{O}(t^{-3p}) = \mathcal{O}(t^{-3p}).
    \end{align*}
    Besides, we have:
    \begin{align*}
        \text{Term B} & \le \rho \|y_{\mu_{t+1}}^*(x_t) - y_{t+1}\| \cdot \|v_{\mu_{t+1}}^*(x_t)\| + \ell (\|v_{\mu_{t+1}}^*(x_t) - v_{\mu_t}^*(x_t)\| + \sqrt{E_{t+1}}), \\
        \implies \text{Term B} & = \mathcal{O}(t^{-3p} \cdot t^p) + \mathcal{O}(t^{p-1}) + \mathcal{O}(\sqrt{E_{t+1}}) = \mathcal{O}(t^{-2p} + t^{p-1} + \sqrt{E_{t+1}}).
    \end{align*}
    Thus, we combine them to get:
    \begin{equation*}
        \|\nabla\Phi_{\mu_{t+1}}(x_t) - \widehat{\nabla}\Phi_{\mu_t}(x_t)\|^2 = \mathcal{O}(t^{-4p} + t^{2p-2} + E_{t+1}).
    \end{equation*}

    Also, we consider the the gap $|\Phi_{\mu_{t+1}}(x_t) - \Phi_{\mu_t}(x_t)|$. According to \Cref{ass:smoothness}, we have:
    \begin{align*}
        |\Phi_{\mu_{t+1}}(x_t) - \Phi_{\mu_t}(x_t)| & = | f(x_t, y_{\mu_{t+1}}^*(x_t)) - f(x_t, y_{\mu_t}^*(x_t)) | \le \nu \| y_{\mu_{t+1}}^*(x_t) - y_{\mu_t}^*(x_t)\|, \\
        \implies |\Phi_{\mu_{t+1}}(x_t) - \Phi_{\mu_t}(x_t)| & \overset{\flat}{=} \mathcal{O}(\sigma_t^2) |\mu_t - \mu_{t+1}| = \mathcal{O} \big( (t+1)^{-3p-1} \big),
    \end{align*}
    where $\flat$ can be implied by Part $\mathcal{A}$.

    We can introduce positive constants $C_1, C_1'$ to construct the following result:
    \begin{align*}
        & \Phi_{\mu_{t+1}}(x_{t+1}) - \Phi_{\mu_t}(x_t) \\
        = & [\Phi_{\mu_{t+1}}(x_{t+1}) - \Phi_{\mu_{t+1}}(x_t)] + [\Phi_{\mu_{t+1}}(x_t) - \Phi_{\mu_t}(x_t)] \\
        \le & -\frac{\alpha_t}{2}\|\nabla\Phi_{\mu_t}(x_t)\|^2 -\frac{\alpha_t}{4}\|\widehat{\nabla}\Phi_{\mu_t}(x_t)\|^2 + C_1\alpha_t E_{t+1} + C_1'\big(\alpha_tt^{-4p} + \alpha_tt^{2p-2} + t^{-3p-1}\big).
    \end{align*}

    \textbf{Part $\mathcal{D}$.} Recall Part $\mathcal{B}$, with positive constants $C_2, C_2'$, we have:
    \begin{equation*}
        E_{t+1} \le \frac{3}{4}E_t + C_2 t^{6p-2q}\| \widehat{\nabla}\Phi_{\mu_{t-1}}(x_{t-1}) \|^2 + C_2' t^{-2p}.
    \end{equation*}
    
    We construct the following Lyapunov function:
    \begin{equation*}
        V_t = \Phi_{\mu_t}(x_t) + M\alpha_t E_t,
    \end{equation*}
    where $M=3C_1$. Then, we have:
    \begin{align*}
        V_{t+1} - V_t & = \Phi_{\mu_{t+1}}(x_{t+1}) - \Phi_{\mu_t}(x_t) + M(\alpha_{t+1} E_{t+1} - \alpha_t E_t) \\
        & \le -\frac{\alpha_t}{2}\|\nabla\Phi_{\mu_t}(x_t)\|^2 -\frac{\alpha_t}{4}\|\widehat{\nabla}\Phi_{\mu_t}(x_t)\|^2 + C_1'\big(\alpha_tt^{-4p} + \alpha_tt^{2p-2} + t^{-3p-1}\big) \\
        & \,\,\,\,\,\, + (C_1\alpha_t + M\alpha_{t+1}) E_{t+1} - M\alpha_t E_t \\
        & \le -\frac{\alpha_t}{2}\|\nabla\Phi_{\mu_t}(x_t)\|^2 -\frac{\alpha_t}{4}\|\widehat{\nabla}\Phi_{\mu_t}(x_t)\|^2 + C_1'\big(\alpha_tt^{-4p} + \alpha_tt^{2p-2} + t^{-3p-1}\big) \\
        & \,\,\,\,\,\, + (C_1\alpha_t + M\alpha_{t+1}) \left( \frac{3}{4}E_t + C_2 t^{6p-2q}\| \widehat{\nabla}\Phi_{\mu_{t-1}}(x_{t-1}) \|^2 + C_2' t^{-2p} \right)  - M\alpha_t E_t \\
        & \overset{\flat}{\le} -\frac{\alpha_t}{2}\|\nabla\Phi_{\mu_t}(x_t)\|^2 -\frac{\alpha_t}{4}\|\widehat{\nabla}\Phi_{\mu_t}(x_t)\|^2 + \underbrace{4C_1C_2\alpha_t t^{6p-2q}}_{\eta_t} \|\widehat{\nabla}\Phi_{\mu_{t-1}}(x_{t-1}) \|^2 \\
        & \,\,\,\,\,\, + \underbrace{C_1''\big(t^{-4p-q} + t^{2p-2-q} + t^{-2p-q} + t^{-3p-1}\big)}_{R_t},
    \end{align*}
    where $\flat$ is due to $M=3C_1$ and $\alpha_t \ge \alpha_{t+1}$, and $C_1''$ is another constant. When $q\ge 3p$, we can select proper $\alpha_t$ to ensure $\eta_{t+1} \le \frac{1}{8}\alpha_t$, and denote $T_0 = \lceil\frac{T}{2}\rceil$. Then, we can get:
    \begin{align*}
        \sum_{t=T_0}^T \frac{\alpha_t}{2}\|\nabla\Phi_{\mu_t}(x_t)\|^2 & \le \sum_{t=T_0}^T \frac{\alpha_t}{2}\|\nabla\Phi_{\mu_t}(x_t)\|^2 - \sum_{t=T_0}^T \frac{\alpha_t}{8}\|\widehat{\nabla}\Phi_{\mu_t}(x_t)\|^2 + \eta_1 \|\widehat{\nabla}\Phi_{\mu_{T_0}}(x_0)\|^2 \\
        & \le V_{T_0} - V_{T+1} + \eta_1 \|\widehat{\nabla}\Phi_{\mu_{T_0}}(x_0)\|^2 + \sum_{t=T_0}^T R_t.
    \end{align*}

    Noting that $V_{T+1}$ is bounded below by $\inf \Phi$, the constant terms can be bounded by $\mathcal{O}(1)$. For the residual sum $\sum_{t=T_0}^T R_t$, the dominant term is $\mathcal{O}(t^{-q-2p})$. Integrating this term yields $\sum_{t=T_0}^T R_t = \mathcal{O}(T^{1-q-2p})$. Dividing both sides by $\sum_{t=T_0}^T \alpha_t = \Theta(T^{1-q})$, we finally get:
    \begin{align*}
        \min_{1 \le t \le T} \|\nabla\Phi_{\mu_t}(x_t)\|^2 & = \frac{\mathcal{O}(1) + \mathcal{O}(T^{1-q-2p})}{\Theta(T^{1-q})} = \mathcal{O}\left( \frac{1}{T^{1-q}} + \frac{1}{T^{2p}} \right).
    \end{align*}
    When selecting $p=1/5$ and $q = 3/5$, and noting $\min_{T_0 \le t \le T}\epsilon_t = \mathcal{O}(T^{-\frac{4}{5}})$ the convergence rate is $\min_{T_0 \le t \le T}\Pi(x_t,y_t) = \mathcal{O}(T^{-\frac{2}{5}})$.
\end{proof}

Given \Cref{thm:main}, we can immediately obtain the following result:
\begin{corollary}
    Whenever the lower-level local minimum satisfies the local curvature condition $\nabla_{yy}^2 g(x,y)\succeq\mu_0I_q$ for some $\mu_0>0$, the bias bound yields $\|\nabla \Phi_\mu(x) - \nabla \Phi_0(x)\| = \mathcal{O}(\mu)$. Combining this bound with the parameter selection in \Cref{thm:main} gives:
    \begin{equation*}
        \min_{\lceil T/2 \rceil \le t \le T} \|\nabla \Phi_0(x_t)\|^2 = \mathcal{O}(T^{-2/5}).
    \end{equation*}
\end{corollary}

\section{Details of Illustrative Example}\label{sec:app-example}

To intuitively show the rationale of our SOSP-based relaxed surrogate, we consider the following concrete example with $x,y\in\mathbb{R}$:
\begin{align*}
    & \min_{x\in\mathbb{R}^p,y\in\mathbb{R}^q} f(x,y) = (x-2)^2 + y^2\\
    \text{subject to } &\,\, y \in \mathcal{S}(x) := \argmin_y g(x,y) = \frac{1}{4}y^4 - \frac{1}{2}xy^2.
\end{align*}
As mentioned earlier, to tackle this nonconvex lower-level BLO problem, existing works instead solve a surrogate lower-level problem by finding a first-order stationary point, i.e., a point where $\nabla_y g(x,y)=0$, or by solving its equivalent formulations under certain other assumptions \citep{xiao2023generalized,shen2023penalty,liu2024moreau,ma2026sun}.

Specifically, the required derivatives are:
\begin{align*}
    \nabla_y g(x,y) = y^3-xy = y(y^2-x), \,\,\,\,\,\, \nabla_{yy}^2 g(x,y) = 3y^2-x.
\end{align*}
Setting the gradient to $0$, we obtain the following FOSP set:
\begin{align*}
    \{(x, 0), (x', \sqrt{x'}), (x', -\sqrt{x'}): \forall x\in\mathbb{R}, x'\ge 0\}.
\end{align*}
When substituting $\nabla_y g(x,y)=0$ into the upper-level problem, we obtain the following solutions: $(x^*, y^*) \in \{(2,0), (1.5, \sqrt{1.5}), (1.5, -\sqrt{1.5})\}$, where $(2,0)$ is a saddle point of $g(x,\cdot)$, and the other two serve as the true solutions. As shown in \Cref{fig:EXAMPLE}, these aforementioned methods can be entirely drawn toward the saddle point, failing to solve the original BLO problem. In contrast, our \alg escapes the saddle point and finds the optimal solution.

\section{Scope and Implications of SOSP Guarantees}\label{sec:app-saddle}

While we established a finite-time convergence rate of $\mathcal{O}(T^{-\frac{2}{5}})$ for \alg in \Cref{sec:analysis}, it remains important to clarify the class of solutions to which the method converges. In this section, we discuss our SOSP-based solutions in detail.

\textbf{The theoretical guarantee is to achieve some SOSP solution.} Our theoretical results guarantee that \alg reaches a \textit{lower-level SOSP} and thereby avoids strict saddle points. However, they do \textit{not} guarantee the selection of the ``best'' solution among multiple lower-level local minima. Although the objective values of any given collection of local minima can be compared directly, guaranteeing the selection of the best solution among all local minima is equivalent to finding a global minimum of the lower-level problem, which is computationally intractable for general nonconvex problems (NP-hard). Therefore, without additional structural assumptions, the effects of different lower-level local solutions on the upper-level problem cannot be characterized in general.

\textbf{An SOSP provides a strictly stronger stationarity guarantee than an FOSP.} The lower-level guarantee that \alg reaches an SOSP does \textit{not} rely on the strict-saddle property. Rather, this property is used only to conclude that an SOSP corresponds to a lower-level local minimum. Even when strictness does not hold, an SOSP still excludes any stationary point that admits a direction of negative curvature and therefore provides a \textit{strictly stronger} stationarity guarantee than an FOSP. This advantage is also well documented in the literature \citep{dauphin2014identifying,allen2018neon2}. This distinction is particularly important in BLO: as shown in \Cref{fig:EXAMPLE}, FOSP-based methods can become trapped at lower-level saddle points and consequently converge to incorrect upper-level solutions, whereas \alg escapes these saddle points and finds the correct solution.

\textbf{Empirical evidence shows the prevalence of strict saddles in practice.} We conducted an experiment using the same LLM-based setting as in \Cref{sec:app-exp}. Specifically, we identified all the saddle points encountered along each trajectory and computed the minimum eigenvalue of the lower-level Hessian at each of them. As shown in \Cref{tab:strict-saddle}, across five repeated trials, all $61$ encountered saddle points exhibited a negative minimum eigenvalue, providing empirical evidence that the saddles encountered in these runs are strict. This implies that most saddle points are strict in practice, demonstrating the applicability of our theories.

\begin{table}[t]
\centering
\small
\begin{tabular}{@{}lrrrr@{}}
\toprule
\textbf{Trajectory} &
\textbf{Candidates} &
\textbf{Strict} &
\textbf{Non-strict} &
\textbf{Undetermined} \\
\midrule
1 & 8  & 8  & 0 & 0 \\
2 & 20 & 19 & 0 & 1 \\
3 & 7  & 7  & 0 & 0 \\
4 & 11 & 10 & 0 & 1 \\
5 & 18 & 17 & 0 & 1 \\
\midrule
\textbf{Total} & \textbf{64} & \textbf{61} &
\textbf{0} & \textbf{3} \\
\bottomrule
\end{tabular}
\caption{Numerical classification of the perturbation-triggered
lower-level stationary-point candidates encountered along five \alg trajectories. \textbf{Undetermined} refers to candidates that could not be classified conclusively because of numerical instability.}
\label{tab:strict-saddle}
\end{table}

\end{document}